%% file: main.tex
\documentclass[letterpaper, twocolumn]{article}
\usepackage{times}  %
\usepackage{helvet}  %
\usepackage{courier}  %
\usepackage[hyphens]{url}  %
\usepackage{graphicx} %
\usepackage{natbib}  %
\usepackage{caption} %
\usepackage{algorithm}
\usepackage{algorithmic}
\usepackage{bm} %

\usepackage{newfloat}
\usepackage{listings}
\DeclareCaptionStyle{ruled}{labelfont=normalfont,labelsep=colon,strut=off} %
\floatstyle{ruled}
\newfloat{listing}{tb}{lst}{}
\floatname{listing}{Listing}
\usepackage{algorithm}
\usepackage{algorithmic}
\usepackage{multirow}

\usepackage{amsmath}
\usepackage{amssymb}
\usepackage{amsfonts}
\usepackage{booktabs}
\usepackage[table]{xcolor}
\usepackage{cleveref}

\usepackage{adjustbox}
\usepackage{appendix}
\usepackage{graphicx}
\usepackage{subcaption}

\usepackage{booktabs}     %
\usepackage{longtable}    %
\usepackage{tabularx}     %
\usepackage{caption}      %

\usepackage{aaai2026}  %

\definecolor{Gray}{gray}{0.9}

\crefname{figure}{Fig.}{Figs.}
\crefname{table}{Tab.}{Tabs.}

\newtheorem{definition}{Definition}
\newtheorem{theorem}{Theorem}

\newtheorem{proof}{Proof}[section]

\ExplSyntaxOn

\NewDocumentCommand { \ie } { }
  {
    \textit{i.e.}
    \peek_meaning_ignore_spaces:NTF .
      { \skip_horizontal:n { -.3ex } \use_none:n }
      {
        \peek_meaning_ignore_spaces:NF ,
          { \skip_horizontal:n { -.3ex } }
      }
  }
\NewDocumentCommand { \eg } { }
  {
    \textit{e.g.}
    \peek_meaning_ignore_spaces:NTF .
      { \skip_horizontal:n { -.3ex } \use_none:n }
      {
        \peek_meaning_ignore_spaces:NF ,
          { \skip_horizontal:n { -.3ex } }
      }
  }

\NewDocumentCommand { \etc } { }
  {
    etc.
    \peek_meaning_ignore_spaces:NT .
      { \use_none:n }
  }

\NewDocumentCommand { \cad } { }
  {
    \textit{c-à-d.}
    \peek_meaning_ignore_spaces:NTF .
      { \skip_horizontal:n { -.3ex } \use_none:n }
      {
        \peek_meaning_ignore_spaces:NF ,
          { \skip_horizontal:n { -.3ex } }
      }
  }
\ExplSyntaxOff

\title{Cross-Domain Few-Shot Learning via Multi-View Collaborative Optimization
with Vision-Language Models}

\author {
    Dexia Chen\textsuperscript{\rm 1},
    Wentao Zhang\textsuperscript{\rm 1},
    Qianjie Zhu\textsuperscript{\rm 2},
    Ping Hu\textsuperscript{\rm 3},
    Weibing Li\textsuperscript{\rm 1},
    Tong Zhang\textsuperscript{\rm 4},
    Ruixuan Wang\textsuperscript{\rm 1}
}
\affiliations {
    \textsuperscript{\rm 1}School of Computer Science and Engineering, Sun Yat-sen University, Guangzhou, China\\
    \textsuperscript{\rm 2}School of Computer, Electronics and Information, Guangxi University, Nanning, China\\
    \textsuperscript{\rm 3}School of Information Science and Engineering, Xinjiang University, Urumqi, China\\
    \textsuperscript{\rm 4}Peng Cheng Laboratory, Shenzhen, China\\
    \{chendx27, zhangwt65, liwb53, wangruix5\}@mail2.sysu.edu.cn,
    2313394044@st.gxu.edu.cn,
    huping@xju.edu.cn,
    zhangt02@pcl.ac.cn
}

\begin{document}

\nocopyright
\maketitle

\input{sec/0_abstract}

\input{sec/1_intro}
\input{sec/2_related}
\input{sec/3_method}
\input{sec/4_exp}

\input{sec/5_conclusion}

{
    \small
    \bibliography{aaai2026}
}

\clearpage

\appendix

\twocolumn[
  \begin{center}
    \LARGE\bfseries Cross-Domain Few-Shot Learning via Multi-View Collaborative Optimization
with Vision-Language Models \par 
    \vspace{1.5em} %
    \large Appendix \par
    \vspace{2em} %
  \end{center}
]

\input{sec/X_suppl}

\end{document}

%% file: sec/0_abstract.tex
\begin{abstract}
\label{sec:abs}
Vision-language models (VLMs) pre-trained on natural image and language data, such as CLIP, have exhibited significant potential in few-shot image recognition tasks, leading to development of various efficient transfer learning methods. These methods exploit inherent pre-learned knowledge in VLMs and have achieved strong performance on standard image datasets. 
However, their effectiveness is often limited when confronted with cross-domain tasks where imaging domains differ from natural images.
To address this limitation, we propose \textbf{Co}nsistency-guided \textbf{Mu}lti-view \textbf{Co}llaborative Optimization (\textbf{CoMuCo}), a novel fine-tuning strategy for VLMs. 
This strategy employs two functionally complementary expert modules to extract multi-view features, while incorporating prior knowledge-based consistency constraints and information geometry-based consensus mechanisms to enhance the robustness of feature learning.
Additionally, a new cross-domain few-shot benchmark is established to help comprehensively evaluate methods on imaging domains distinct from natural images.
Extensive empirical evaluations on both existing and newly proposed benchmarks suggest CoMuCo consistently outperforms current methods in few-shot tasks. 
The code and benchmark will be released.
\end{abstract}

%% file: sec/1_intro.tex
\section{Introduction}

Current deep learning methods often require vast amounts of labeled data which may be prohibitively expensive and difficult to obtain in domains such as rare disease diagnosis and industrial defect detection. To address this challenge, various few-shot learning techniques~\cite{metalearning1, metalearning2} have been developed to enable models to learn effectively from limited data. While previous methods are effective in some scenarios, they generally suffer from limited generalization ability. %

In recent years, the emergence of pre-trained vision-language models (VLMs)~\cite{softclip, StructureCLIP, align, wei2024efficient, flip}, especially CLIP~\cite{CLIP}, offer new solutions for few-shot learning.
An image encoder and a text encoder are commonly included in these models to align image features and text embeddings. The alignment is facilitated by enhancing the cosine similarity of the corresponding image-text pairs. 
After being pre-trained on large amounts of data, these models acquire strong semantic understanding and effective zero-shot image recognition abilities. The powerful feature representation ability and open-vocabulary recognition ability effectively alleviate the problems faced by few-shot learning. 
To enable efficient transfer learning of pre-trained VLMs in few-shot scenarios, a series of fine-tuning techniques have been proposed, \eg, methods based on prompt tuning~\cite{CoOp, CoCoOp, plot, maple, progard} or 
adapter tuning~\cite{Clip-adapter, tip-adapter, huang2024lp++}. 

\begin{figure}
    \centering
    \includegraphics[width=0.45\textwidth]{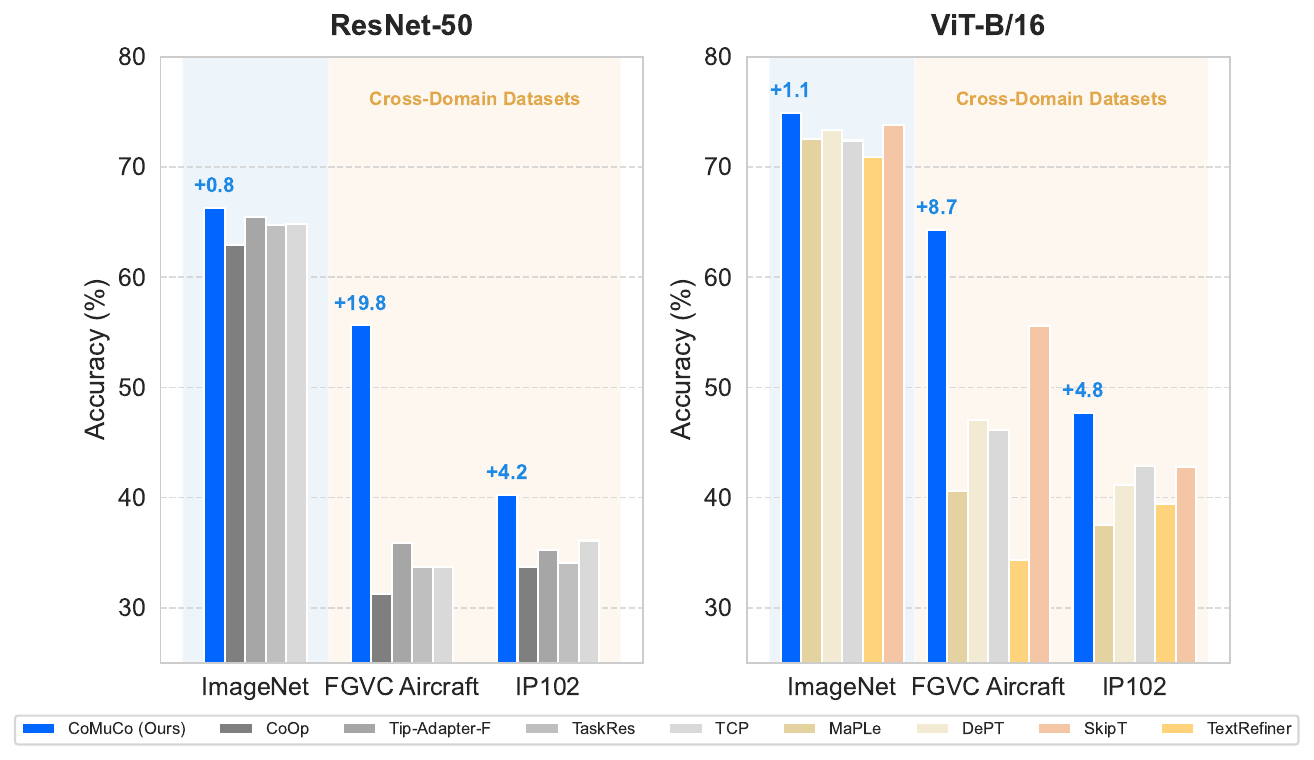}
    \caption{Accuracy (\%) comparison on in-domain (ImageNet) and cross-domain (FGVC Aircraft \& IP102) datasets under the 16-shot setting using ResNet-50 (left) and ViT-B/16 (right). The ‘+’ marks the performance gain of our method over the strongest baseline.
    }
    \label{fig:radar}
\end{figure}

As illustrated in \cref{fig:radar}, since these methods are initially designed to leverage intrinsic knowledge in VLMs, their performance is inherently dependent on the alignment between pre-learned knowledge in VLMs and the to-be-learned knowledge in the downstream task. %
Strong alignment typically results in better performance, whereas in cross-domain settings with substantial domain shifts, the reduced alignment significantly limits their effectiveness.
Furthermore, simple fine-tuning strategies may only assimilate a subset of discriminative features present in the training dataset, while comprehensive discriminative characteristics remain inadequately captured~\cite{ensemble}, thus constraining model performance especially on cross-domain datasets.

To address the challenges of applying VLMs to few-shot learning in cross-domain scenarios and enhance the extraction of discriminative features with VLMs, we propose CoMuCo, a consistency-guided multi-view collaborative optimization framework. 
By establishing diverse learning preferences, this framework effectively facilitates the acquisition of multi-view features.
Specifically, our framework consists of two functionally complementary expert modules, i.e., a Feature Integrator, which extracts and refines knowledge relevant to cross-domain classification from pre-trained models, and a Feature Refiner, which actively learns task-specific features from cross-domain data. 
Both modules are governed by a consensus constraint based on information geometry theory, which promotes the learning of mutually compatible and robust feature representations.
Additionally, a prior consistency constraint is implemented to preserve logits consistency across the fine-tuning process by constraining logits deviations to follow a Laplacian prior distribution, thereby effectively mitigating catastrophic forgetting of general knowledge.

Furthermore, recent efficient transfer learning methods commonly adopt the CLIP Benchmark\footnote{The CLIP Benchmark~\cite{CoOp} consists of 11 widely used datasets for evaluating few-shot learning.} for performance evaluation.
However, many of its datasets have substantial domain overlap with CLIP's pretraining corpus. While datasets such as DTD~\cite{DTD} and EuroSAT~\cite{eurosat} provide some cross-domain variation, their diversity remains limited.
To enable a more comprehensive evaluation across distinct visual domains, we curated a set of datasets that differ significantly from natural images and proposed a new cross-domain few-shot benchmark. Our method was evaluated on both the CLIP Benchmark and our proposed benchmark, consistently achieving state-of-the-art performance.
The main contributions of this study are summarized below. %
\begin{itemize}
\item A novel \textbf{Co}nsistency-guided \textbf{Mu}lti-view \textbf{Co}llaborative Optimization (\textbf{CoMuCo}) 
is proposed to effectively learn knowledge from downstream task data in few-shot scenarios, especially on cross-domain tasks.
\item A prior consistency constraint is proposed, achieving the preservation of prior knowledge by constraining logits drift to satisfy the Laplace distribution.
\item A novel multi-view geodesic consensus mechanism is proposed to facilitate the learning of more robust discriminative representations.
\item Extensive empirical evaluations were performed on both the existing benchmark and the cross-domain benchmark, with SOTA performance achieved by CoMuCo.

\end{itemize}

%% file: sec/2_related.tex
\section{Related Work}
\begin{figure*}[t]
    \centering
    \includegraphics[width=0.8\textwidth]{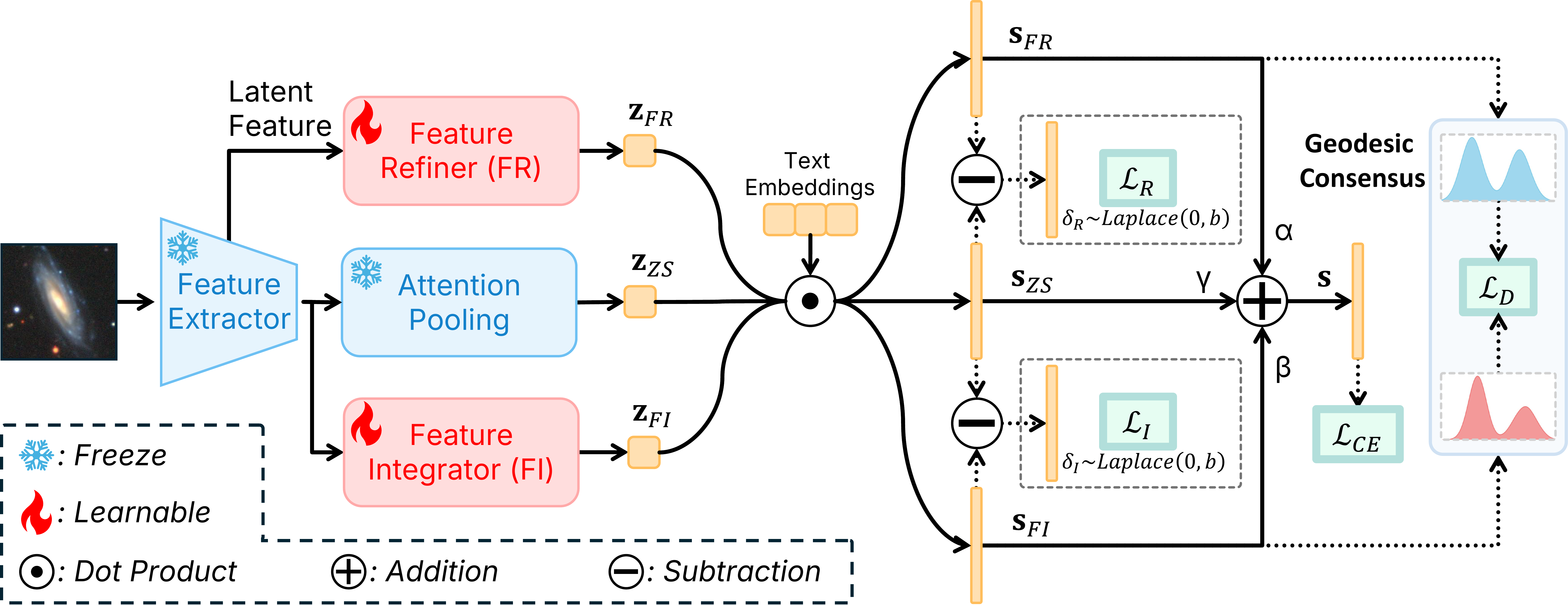}

    \caption{
    Overview of CoMuCo. The proposed framework is constructed around two core modules: the Feature Integrator (FI) and the Feature Refiner (FR). 
    A consensus constraint aligns FI and FR for enhanced feature learning, while a prior consistency constraint regulates logit deviation to preserve zero-shot knowledge. Feature vectors $\textbf{z}_{FI}$, $\textbf{z}_{FR}$, and $\textbf{z}_{ZS}$ are extracted from FI, FR, and frozen CLIP modules, with corresponding logits $\textbf{s}_{FI}$, $\textbf{s}_{FR}$, and $\textbf{s}_{ZS}$ obtained through class text embedding alignment.
    ``Attention Pooling" is configured at the final transformer block in the ViT architecture.
    }
    \label{fig:overview}
\end{figure*}

\noindent\textbf{Vision-Language Models}
The recently developed pre-trained VLMs~\cite{softclip, StructureCLIP, align, wei2024efficient, flip, siglip, siglip2, evaclip, cherti2023reproducible, metaclip} have been widely applied to few-shot learning.
Among these VLMs, CLIP~\cite{CLIP} has garnered significant attention for its 
generalization capability for downstream tasks. 
CLIP is pre-trained on a vast number of image-text pairs, learning the semantic relationships between images and text, which enables it to extract visual features rich in semantic information. This powerful pretraining capability renders CLIP a promising base model for transfer learning. 

CLIP consists of two primary components, \ie, an image encoder and a text encoder. Specifically, for a query image and $C$ class categories, with each category associated with a textual sentence, CLIP firstly extracts the image feature $\mathbf{z}\in \mathbb{R}^{d}$ and text embeddings $\mathbf{t} \in \mathbb{R}^{C\times d}$ via the image encoder and the text encoder, respectively, where $d$ represents the feature dimension. Then, the similarity between the image feature and each text embedding is computed, resulting in a similarity vector $\mathbf{s} = {sim}(\mathbf{z}, \mathbf{t}) \in \mathbb{R}^{C}$, where $sim(\cdot, \cdot)$ represents the cosine similarity measurement. Consequently, the  probability output is given by $\mathbf{p} = {softmax}({\mathbf{s}/\tau})$, where $\tau$ is the temperature coefficient. 
The class with the highest similarity score is selected as the prediction. 

\noindent\textbf{Efficient Transfer Learning}
To fully leverage the pre-trained knowledge of VLMs in few-shot learning scenarios, a series of efficient transfer learning methods have been developed~\cite{CoOp, plot, calip, huang2024lp++, maple, yao2024tcp, zhang2024dept, CoCoOp, tip-adapter, taskres}. These methods can be primarily divided into two groups, prompt-tuning~\cite{plot, maple, yao2024tcp, CoOp, CoCoOp, zhang2024dept} and adapter-tuning~\cite{huang2024lp++, tip-adapter, Clip-adapter, taskres}. 
Prompt-tuning methods like CoOp~\cite{CoOp} use learnable prompts in the text encoder, while adapter-tuning methods such as CLIP-Adapter~\cite{Clip-adapter} add lightweight modules to encoders.
Although these methods demonstrate robust performance, their model adaptations are predominantly constrained to input tokens and output features, thereby limiting their effectiveness in cross-domain scenarios.
By introducing two complementary expert modules alongside prior constraints and information-geometric consensus constraints, our method markedly mitigates this issue and improves the model’s generalization ability.

%% file: sec/3_method.tex
\section{Method}

\subsection{Overview}
To facilitate the comprehensive learning of discriminative features from downstream task data, CoMuCo, a consistency-guided multi-view collaborative optimization framework, is introduced. 
As illustrated in \cref{fig:overview}, this framework facilitates the learning of features from different perspectives by incorporating two functionally complementary modules: the Feature Integrator (FI) and the Feature Refiner (FR). Specifically,  FI is designed to extract and refine knowledge pre-learned by the VLM that remains relevant to the downstream classification task, whereas FR actively learns novel task-specific knowledge from downstream data. 
To prevent excessive forgetting of the pre-trained model's general knowledge within each module, a prior consistency constraint is enforced in the logit space. Specifically, the deviation between each module's logits and those of zero-shot CLIP is encouraged to follow a zero-mean Laplacian distribution, thereby promoting minimal modifications to the logits and ensuring that pre-learned knowledge is preserved throughout the training process of each module.
Furthermore, to enhance the robustness of feature learning, we propose a multi-view consensus mechanism grounded in information geometry theory, which approximately minimizes the squared geodesic distance between probability distributions from different perspectives on a statistical manifold, thereby fostering compatibility across views and promoting more robust feature learning.

\begin{figure}[t]
    \centering
    \includegraphics[width=0.43\textwidth]{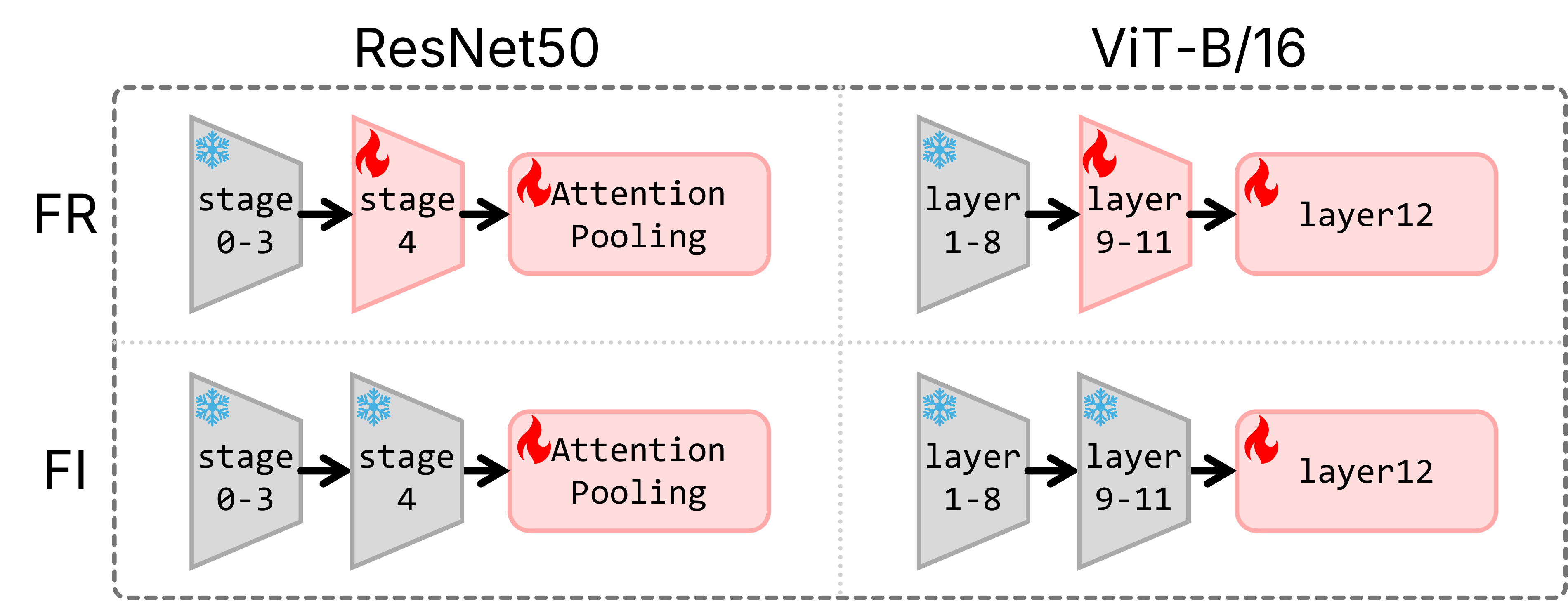}
    \caption{
    Illustrations of FI and FR under different architectures.
    They are initialized with the same architecture and weights as the pre-trained model, with intermediate results from frozen CLIP forward propagation being reused as FI and FR inputs to reduce computation.
    }
    \label{fig:FIFR}
\end{figure}

\subsection{Dual-Expert Framework}
To comprehensively learn discriminative features for downstream tasks, two structurally decoupled and functionally complementary expert modules are introduced. 
As shown in \cref{fig:FIFR}, these modules are designed to implicitly capture features from different perspectives of the downstream task data through distinct fine-tuning strategies:

\begin{itemize}
    \item \textbf{Feature Integrator} (Invariant Expert) is designed to preserve existing knowledge from the VLM through conservative parameter modifications, focusing updates solely on the last module.
    \item \textbf{Feature Refiner} (Adaptive Expert) captures novel patterns induced by downstream task data, employing fine-tuning of deeper network layers to achieve superior adaptation to domain-specific data distributions.
\end{itemize}
By learning features from different perspectives, this dual-expert framework enables effective information extraction.

\subsection{Prior Consistency Constraint}
Since substantial prior knowledge is embedded in pre-trained models during their pre-training phase, its complete erasure during fine-tuning is considered harmful.
To address catastrophic forgetting of prior knowledge triggered by fine-tuning, a prior consistency constraint is implemented to reduce the number of elements modified in the fine-tuning branch's logits when compared to original CLIP logits, thereby necessitating sparsity in the logits offset. 
This sparsity is enforced by requiring that the logits offset follows a sparse prior distribution, modeled using zero-mean Laplacian distribution.

We consider the deviation between the expected logits $\mathbf{s}$ produced by each expert and the zero-shot logits $\mathbf{s}_{\text{ZS}}$ from frozen CLIP.
Let $\bm{\delta} = \mathbf{s} - \mathbf{s}_{ZS}$ denote the deviation vector in the logit space.

\begin{definition}[Laplace Prior on Logits Deviation]
Each component $\delta_c$ of the deviation vector $\bm{\delta}$ is independently drawn from a zero-mean Laplace distribution:
\begin{equation}
p(\delta_c) = \frac{1}{2b} \exp\left(-\frac{|\delta_c|}{b}\right),
\end{equation}
where $b$ is a scale hyper-parameter.
\end{definition}
Then, we demonstrate the negative log-prior can be interpreted as an $L_1$ regularization:
\begin{theorem}[Laplace Prior Equivalence]
Imposing an $L_1$ regularization on the deviation vector $\bm{\delta}$ is equivalent to assuming an independent Laplace prior:
\begin{equation}
-\log p(\bm{\delta}) \propto \frac{1}{b} \|\bm{\delta}\|_1.
\end{equation}
Proof is provided in Appendix A.1
\end{theorem}
Based on this theory, the resulting prior consistency losses are defined as:
\begin{equation}
\mathcal{L}_{R} = \|\mathbf{s}_{FR} - \mathbf{s}_{ZS}\|_1, \quad \mathcal{L}_{I} = \|\mathbf{s}_{FI} - \mathbf{s}_{ZS}\|_1.
\end{equation}
This Laplace-based regularization enforces sparse adaptation during fine-tuning. The model is encouraged to modify its predictions only for a small set of classes while maintaining consistency with the powerful prior for most categories. This mechanism enables local adaptation without erasing general visual knowledge.

\subsection{Multi-view Consensus Constraint}

As the features derived from different perspectives are intended to collaboratively address the same classification task, a consensus between their predictions is expected rather than mutual contradiction. 
To facilitate this, the expected predictive distributions generated by the two expert branches are aligned by minimizing Jeffreys divergence on the statistical manifold $\mathcal{M}_P$ of output probabilities.

\begin{definition}[Jeffreys Divergence]
The Jeffreys divergence is defined as the symmetric form of KL divergence:
\begin{equation}
D_J(\mathbf{p} \| \mathbf{q}) = D_{KL}(\mathbf{p} \| \mathbf{q}) + D_{KL}(\mathbf{q} \| \mathbf{p}).
\end{equation}
\end{definition}
The consensus loss is defined via the Jeffreys divergence:
\begin{equation}
\mathcal{L}_D = \frac{1}{2} D_J(\mathbf{p}_{FR}, \mathbf{p}_{FI}).
\end{equation}
This formulation is grounded in the intrinsic geometry of statistical manifolds. Specifically, we demonstrate that the Jeffreys divergence admits a higher-order approximation to the squared geodesic distance between two probability distributions on such a manifold:
\begin{theorem}[Geodesic Divergence Approximation]
Let $\mathcal{M}$ be a statistical manifold, where points on $\mathcal{M}$ are parameterized by a local coordinate system $\pi$. For any two points $P$ and $Q$ on $\mathcal{M}$, with coordinates $\pi_P$ and $\pi_Q$ respectively, the squared geodesic distance $d^2(P,Q)$ connecting them can be approximated to fourth order through the Jeffreys divergence:
\begin{equation}
D_J(P, Q) = d^2(P, Q) + O(\|\pi_Q - \pi_P\|^4),
\end{equation}
where $\|\pi_Q - \pi_P\|$ represents the norm of the parametric coordinate difference between the two points. (Proof is provided in Appendix A.2)
\end{theorem}
This geometric perspective suggests that minimizing $\mathcal{L}_D$ effectively reduces the geodesic distance between the prediction distributions of the two experts on the statistical manifold, thereby encouraging predictive consensus and enhances the robustness of each expert.

\subsection{Training and Inference}

For each image $\mathbf{x}_i$, the fused logits, used for both training and inference, are obtained by aggregating the logits from FR, FI, and the frozen CLIP, which are denoted as $\mathbf{s}_{FR}(\mathbf{x}_i)$, $\mathbf{s}_{FI}(\mathbf{x}_i)$, and $\mathbf{s}_{ZS}(\mathbf{x}_i)$, respectively:
\begin{equation}
\textbf{s}_i = \alpha \cdot \mathbf{s}_{FR}(\mathbf{x}_i) + \beta \cdot \mathbf{s}_{FI}(\mathbf{x}_i) + \gamma \cdot \mathbf{s}_{ZS}(\mathbf{x}_i),
\end{equation}
The weights $\alpha$, $\beta$, and $\gamma$ are expert coefficients, with $\gamma = 1 - \alpha - \beta$. The cross-entropy loss over the training set serves as the expected likelihood objective:
\begin{equation}
\mathcal{L}_{\text{CE}} = -\frac{1}{N} \sum_{i=1}^N \log p(\mathbf{y}_i | \mathbf{s}_i).
\end{equation}
where $p(y_i | \mathbf{s}_i)$ represents the predicted probability for the ground-truth label $y_i$ given the fused logits $\mathbf{s}_i$.

The complete training objective integrates multiple regularization terms:
\begin{equation}
\mathcal{L} = \underbrace{\mathcal{L}_{\text{CE}}}_{\text{Joint Likelihood}} + \underbrace{\lambda_1 \mathcal{L}_{R} + \lambda_2 \mathcal{L}_{I}}_{\text{Prior Regularization}} + \underbrace{\lambda_3 \mathcal{L}_{D}}_{\text{Consensus Regularization}}.
\end{equation}

%% file: sec/4_exp.tex
\section{Experiments}

\begin{figure*}
    \centering
    \includegraphics[width=0.9\textwidth]{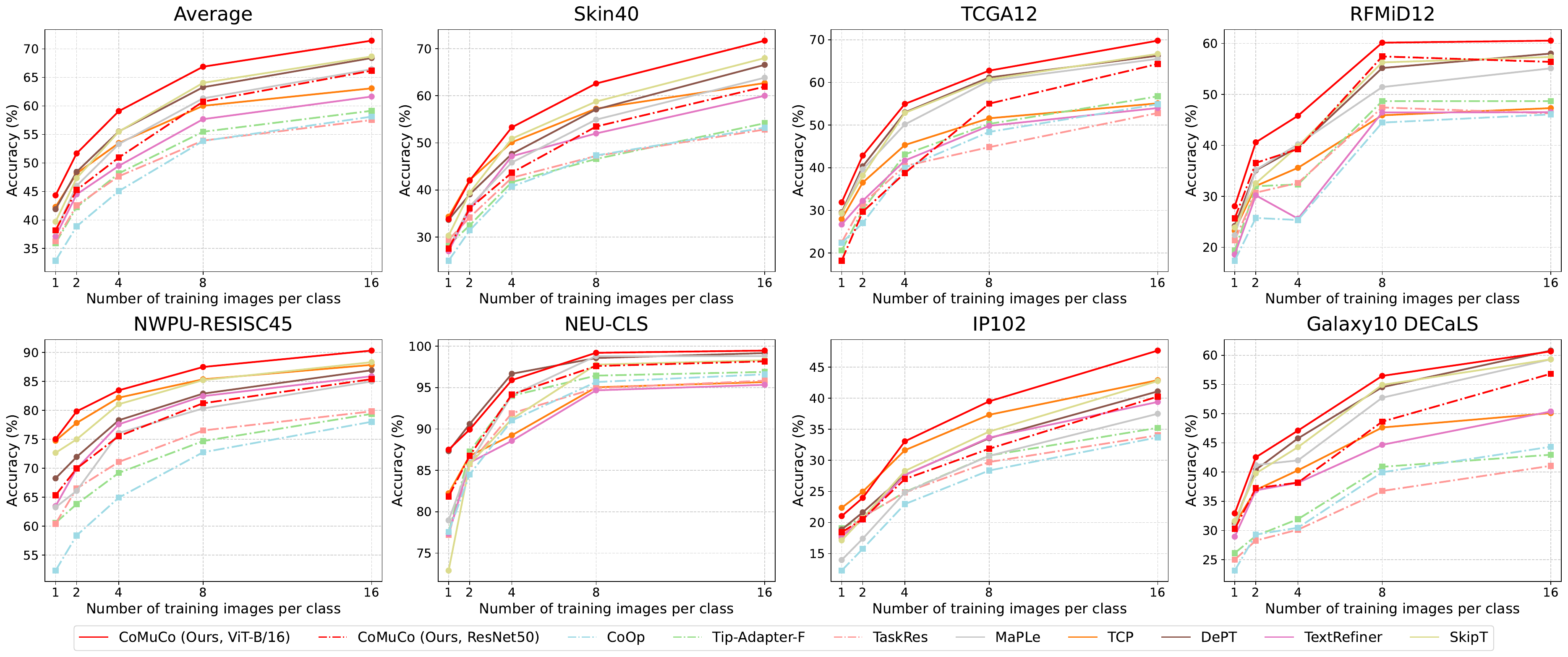}
    \caption{Performance comparison on the cross-domain benchmark. Dashed lines for ResNet50 and solid lines for ViT-B/16.}
    \label{fig:cross_domain_results}
\end{figure*}

\subsection{Experimental Settings}
\noindent\textbf{Datasets} 
To address the CLIP Benchmark’s limitations in evaluating model performance across various visual domains, a novel benchmark is introduced, which incorporates seven diverse image datasets, \ie , 
Skin40~\cite{skin40} with 40 classes of skin disease, TCGA12~\cite{tcga} with 12 classes of tissue pathology, RFMiD12~\cite{RFMiD} with 12 classes of fundus, NWPU-RESISC45~\cite{nwpu} with 45 classes of remote sensing, NEU-CLS~\cite{neucls} with 6 classes of hot-rolled steel defect, IP102~\cite{IP102} with 102 classes of crop pest and disease, and Galaxy10 DECaLS~\cite{galaxy10} with 10 classes of galaxy. 
Collectively, these datasets cover a wide range of fields, enabling more comprehensive assessment of model adaptability.

In addition, the CLIP Benchmark was still used to thoroughly evaluate the model's performance, which includes 11 datasets, \ie, ImageNet~\cite{ImageNet}, Caltech101~\cite{Caltech}, Food101~\cite{food101}, DTD~\cite{DTD}, EuroSAT~\cite{eurosat}, FGVCAircraft~\cite{FGVCAircraft}, Flowers102~\cite{flower102}, OxfordPets~\cite{oxfordpets}, StanfordCars~\cite{stanfordcars}, SUN397~\cite{sun397}, and UCF101~\cite{ucf101}. 
Moreover, ImageNet-Sketch~\cite{imagenetsketch} and ImageNet-V2~\cite{imagenetV2} are incorporated to assess the model's domain generalization capability.

\noindent\textbf{Implementation} 
Following previous studies~\cite{taskres, CoOp},
we trained models under $K$-shot settings ($K=1, 2, 4, 8, 16$) with $K$ images per class and evaluated on the full test set. 
Unless otherwise stated, ResNet-50 was used as the visual backbone, and pre-defined text templates~\cite{taskres} were used for text encoding. 
Training was conducted using SGD with cosine learning rate decay for 50 epochs (300 for cross-domain settings), starting with a warm-up from $1\mathrm{e}{-5}$ to 0.002 in the first epoch. The default batch size was 32. Data augmentations from CoOp~\cite{CoOp} (random crop and flip) were applied. Hyperparameters were fixed as $\alpha=\beta=0.2$, and $\lambda_1=\lambda_2=\lambda_3=0.1$. All results were averaged over three runs with different seeds.

\noindent\textbf{Baselines} To validate the effectiveness of our method, comparisons were made with 
SOTA
efficient transfer learning methods, including CoOp~\cite{CoOp}, Tip-Adapter-F~\cite{tip-adapter}, TaskRes~\cite{taskres}, MaPLe~\cite{maple}, TCP~\cite{yao2024tcp}, DePT~\cite{zhang2024dept}, TextRefiner~\cite{xie2024textrefiner} and SkipT~\cite{wu2025skip}.

\begin{table}[t]
\centering
\small
\setlength{\tabcolsep}{1mm}
\begin{tabular}{cccccccccc}
\toprule
\multicolumn{7}{c}{\textbf{Components}} & \multicolumn{3}{c}{\textbf{Datasets}} \\
\cmidrule(r){1-7} \cmidrule(r){8-10}
$\mathcal{L}_{CE}$ & FI & FR & $\mathcal{L}_{I}$ & $\mathcal{L}_{R}$ & $\mathcal{L}_{D}$ & & ImageNet & Stanford Cars & Galaxy \\
\midrule
 &  &  &  &  &  & & 58.18 & 55.61 & 13.90 \\
\checkmark & \checkmark &  &  &  &  & & 65.40 & 79.27 & 51.23 \\
\checkmark & \checkmark &  & \checkmark &  &  & & 65.50 & 79.53 & 50.10 \\
\checkmark &  & \checkmark &  &  &  & & 64.33 & 81.27 & 53.30 \\
\checkmark &  & \checkmark &  & \checkmark &  & & 65.10 & 83.53 & 56.23 \\
\checkmark & \checkmark & \checkmark &  &  &  & & 64.47 & 81.23 & 53.26 \\
\checkmark & \checkmark & \checkmark &  &  & \checkmark & & 65.73 & 80.23 & 54.43 \\
\checkmark & \checkmark & \checkmark & \checkmark & \checkmark &  & & 65.63 & 83.87 & 55.67 \\
\midrule
\checkmark & \checkmark & \checkmark & \checkmark & \checkmark & \checkmark & & \textbf{66.27} & \textbf{85.07} & \textbf{56.83} \\
\bottomrule
\end{tabular}%
\caption{Ablation study of our method on 3 representative datasets under the 16-shot setting. %
}
\label{tab:ablation}
\end{table}

\subsection{Efficacy of the Proposed Method}

\begin{figure*}
    \centering
    \includegraphics[width=0.9\textwidth]{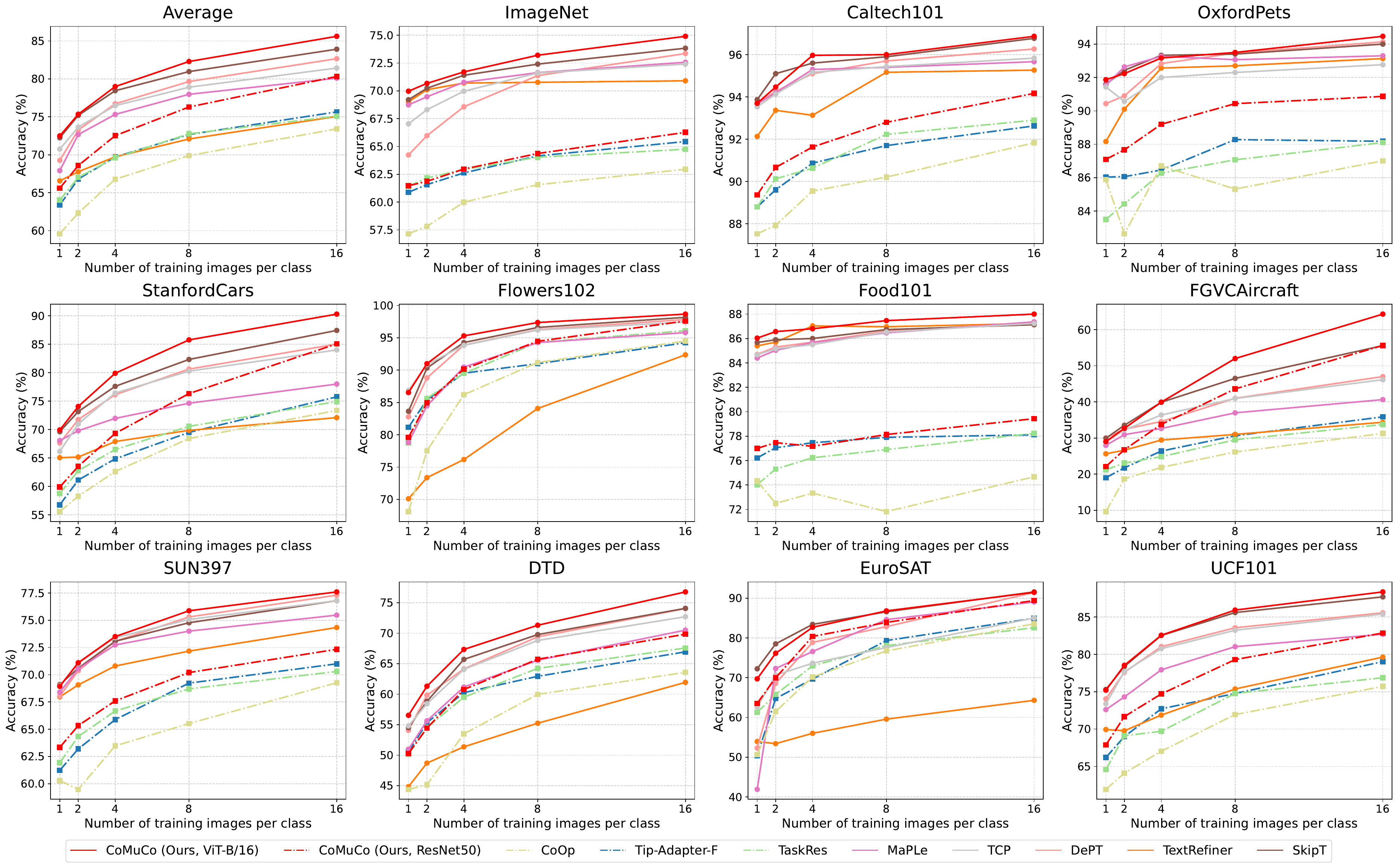}
    \caption{Performance comparison on the CLIP Benchmark. Dashed lines for ResNet50 and solid lines for ViT-B/16.}
    \label{fig:CLIP_benchmark_results}
\end{figure*}

\noindent\textbf{Results on Cross-Domain Few-Shot Benchmark}
Our method was first assessed on the cross-domain few-shot benchmark. 
To confirm the challenges for cross-domain few-shot recognition, zero-shot CLIP as the naive baseline was evaluated on the benchamrk, which revealed that the frozen pre-trained CLIP model is unable to perform effective classification in cross-domain tasks (Appendix C.2).

As shown in \cref{fig:cross_domain_results}, our method CoMuCo consistently outperforms all baselines. With ResNet50 as the visual encoder, it achieves superior results, particularly with a slightly larger number of training images, surpassing the strongest baseline by 5.27\% and 7.03\% under the [8, 16]-shot settings. When ViT-B/16 is used, CoMuCo exhibits superior performance across all competing approaches, yielding improvements of 2.03\%, 3.23\%, 3.59\%, 2.83\%, and 2.78\% over the best baseline under the [1, 2, 4, 8, 16]-shot settings. Notably, even with ResNet50, the performance of CoMuCo remains comparable to certain ViT-B/16-based baselines, exceeding TextRefiner and TCP while matching MaPLe under [8, 16]-shot settings.
These results support that our method can more effectively learn knowledge from the limited training data when the imaging modality of the downstream task is significantly different from those used in CLIP pre-training.

\noindent\textbf{Results on CLIP Benchmark}
As shown in \cref{fig:CLIP_benchmark_results} (top-left subfigure), on the widely used CLIP Benchmark for few-shot learning, our method consistently achieves superior average performance compared to SOTA methods across all the few-shot settings. 
As the number of training samples increases, the performance gap progressively widens. Under the [1, 4, 16]-shot settings, our method outperforms the best baseline by 1.48\%, 2.77\%, and 4.65\% on ResNet50, and by 0.25\%, 0.54\%, and 1.70\% on ViT-B/16, respectively.
Notably, our method excels on the two fine-grained datasets StanfordCars and FGVCAircraft, and on the texture dataset DTD. 
Classification of fine-grained classes often requires more specialized knowledge~\cite{Clip-adapter}, as does classification in the texture classes, which is less frequently encountered during CLIP's pre-training phase.
In this case, our method enables the model to effectively learn from new data, resulting in superior performance. 
Conversely, on the ImageNet, Flowers102, and Food101 datasets, our method performs on par with competing methods, as CLIP's pre-trained knowledge is sufficient to acquire substantial task-specific knowledge with minimal additional learning. 
To conclude, our method exhibits marked superiority in standard few-shot learning tasks, particularly in the domains of fine-grained classification and cross-domain data recognition.

\begin{figure}[tb]
    \centering
    \includegraphics[width=0.9\linewidth]{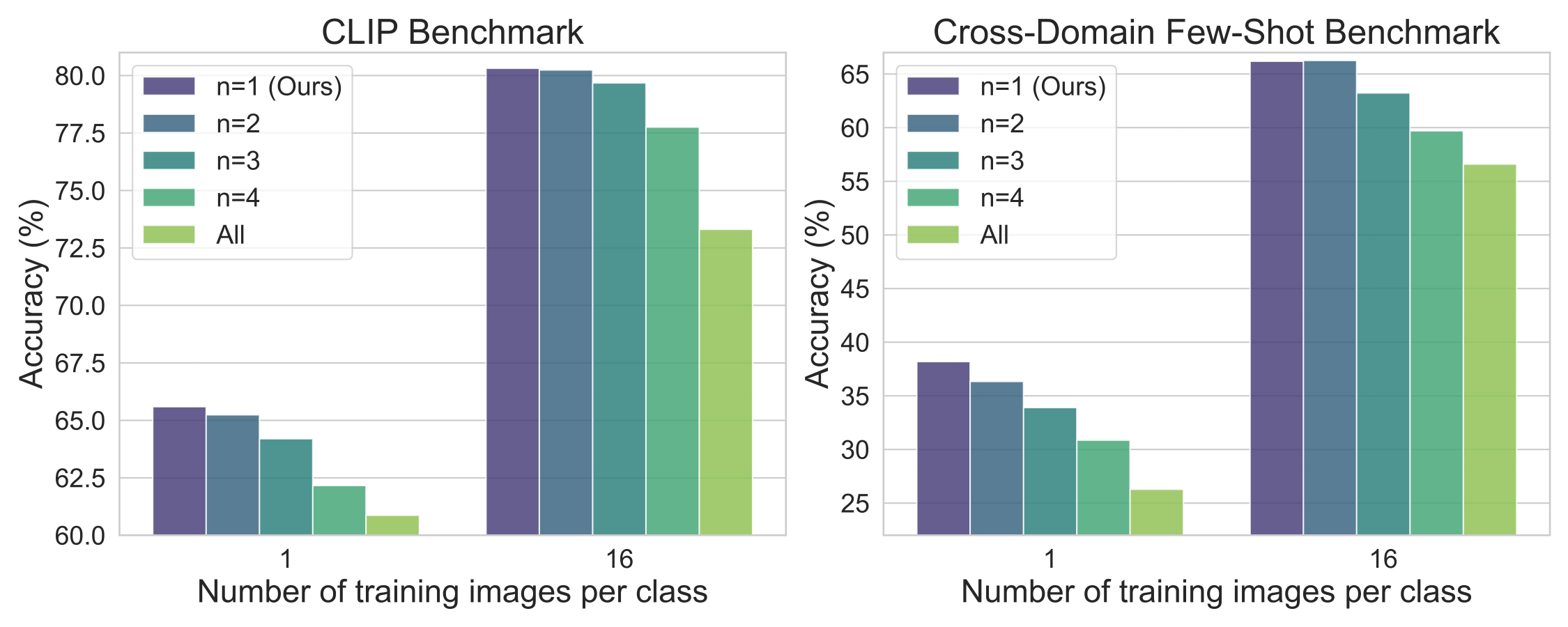}
    \caption{
    Average results on CLIP and cross-domain benchmarks for different FR fine-tuning strategies. Here, `n' indicates the count of fine-tuned layers closest to output.
    }
    \label{fig:clip_fap}
\end{figure}

\subsection{Ablation Study}
\noindent\textbf{Ablation Study on Model Components} 
Ablation studies were performed on ImageNet, Stanford Cars, and Galaxy10 DECaLS (`Galaxy') under the 16-shot setting to evaluate the impact of key components in the proposed framework.
These datasets respectively represent natural image classification, fine-grained classification, and cross-domain classification tasks, allowing for a comprehensive evaluation.
As shown in \cref{tab:ablation}, FR is more advantageous for fine-grained and cross-domain classification, whereas FI excels in natural image classification (row 2 vs. 4). 
This discrepancy arises because FI retains most pre-learned knowledge and efficient learning from natural image data is achieved through the correction of attention pooling.
However, in tasks involving fine-grained distinctions or large domain shifts, FR demonstrates superior results by enabling more comprehensive feature refinement and enhanced knowledge adaptation.
Combining FR and FI (row 6) yields intermediate performance.

Adding Prior Consistency Constraint improves results by 1.99\% for FR (row 4 vs. 5) and 2.07\% for the dual-expert integration (row 6 vs. 8), 
which supports that the prior constraint alleviates the overfitting issue by preventing excessive forgetting of pre-learned knowledge in CLIP. 
Additionally, inclusion of the multi-view consensus constraint enhances the performance of dual-expert integration (row 6 vs. row7).
When prior constraint was employed, the consensus mechanism demonstrated enhanced performance gains (rows 7 \& 8 vs. row 9), indicating that prior constraints assist the consensus mechanism in 
improving the model's generalization capability. 
These ablation results confirm the effectiveness of all CoMuCo components.

\begin{figure*}[tb]
    \centering
    \includegraphics[width=0.9\linewidth]{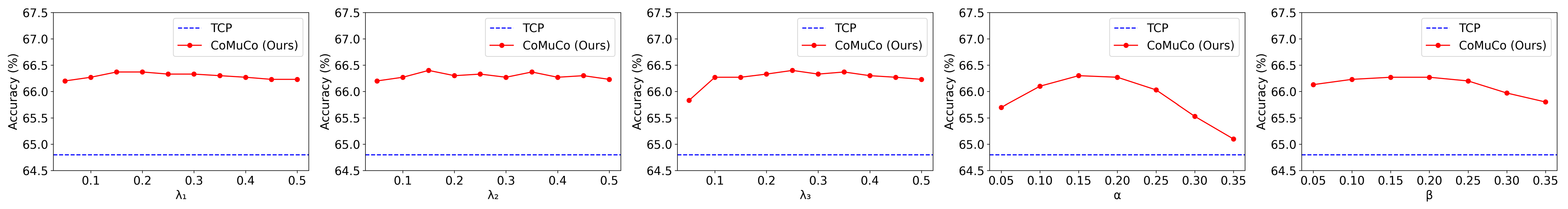}
    \caption{Sensitivity study on ImageNet. TCP is a representative strong baseline.}
    \label{fig:sensitivity}
\end{figure*}

\noindent\textbf{Impact of Fine-tuning Layer Configurations in FR} 
The impact of fine-tuning depths in the FR was evaluated.
\cref{fig:clip_fap} shows that, performance generally declines with more fine-tuned layers, especially with fewer or cross-domain samples.
Under the 16-shot setting, fine-tuning 3 layers or the entire visual encoder reduced average performance by 0.64\% and 7.01\% on the CLIP benchmark and by 3.08\% and 9.18\% on the cross-domain benchmark, relative to tuning the last layer only. In the 1-shot cross-domain case, declines extended to 3.67\% and 11.73\%.
These findings indicate that early layers are more prone to overfitting when data is scarce, and that limiting fine-tuning to the final layer helps preserve generalizable representations learned during pre-training.

\begin{table}[tb]
    \centering
    \small
    \setlength{\tabcolsep}{1mm}
    \begin{tabular}{lcccc}
        \toprule
        \multirow{2}{*}{Method} & \multirow{2}{*}{Visual Backbone} & Source & \multicolumn{2}{c}{Target} \\
        \cmidrule(lr){3-3} \cmidrule(lr){4-5}
        & & ImageNet & -V2 & -Sketch\\
        \midrule
        Zero-Shot CLIP  & \multirow{5}{*}{ResNet-50} & 58.18 & 51.34 & 33.32  \\
        Linear Probe CLIP  & & 55.87 & 45.97 & 19.07 \\
        CoOp  & & 62.95 & 55.11 & 32.74 \\
        TCP & & 64.8 & 56.27 & 34.23 \\
        \rowcolor{Gray}
        CoMuCo & & \textbf{66.27} & \textbf{57.13} & \textbf{35.03}\\
        \midrule
        Zero-Shot CLIP  & \multirow{5}{*}{ResNet-101} & 61.62 & 54.81 & 38.71 \\
        Linear Probe CLIP  & & 59.75 & 50.05 & 26.80 \\
        CoOp  & & 66.60 & 58.66 & 39.08 \\
        TCP & & 67.53 & 59.10 & 40.37 \\
        \rowcolor{Gray}
        CoMuCo & & \textbf{69.30} & \textbf{60.60} & \textbf{41.93} \\
        \midrule
        Zero-Shot CLIP  & \multirow{5}{*}{ViT-B/32} & 62.05 & 54.79 & 40.82 \\
        Linear Probe CLIP  & & 59.58 & 49.73 & 28.06 \\
        CoOp  & & 66.85 & 58.08 & 40.44 \\
        TCP & & 67.73 & 58.50 & 41.50 \\
        \rowcolor{Gray}
        CoMuCo & & \textbf{69.60} & \textbf{60.17} & \textbf{42.57} \\
        \midrule
        Zero-Shot CLIP  & \multirow{5}{*}{ViT-B/16} & 66.73 & 60.83 & 46.15 \\
        Linear Probe CLIP  & & 65.85 & 56.26 & 34.77 \\
        CoOp  & & 71.92 & 64.18 & 46.71 \\
        TCP & & 72.40 & 64.83 & 48.17 \\
        \rowcolor{Gray}
        CoMuCo & & \textbf{74.90} & \textbf{66.80} & \textbf{49.47} \\
        \bottomrule
    \end{tabular}
    \caption{Performance in domain adaption and with different CLIP visual backbones.
    }
    \label{tab:DG}
\end{table}

\begin{figure}[tb]
    \centering
    \includegraphics[width=0.9\linewidth]{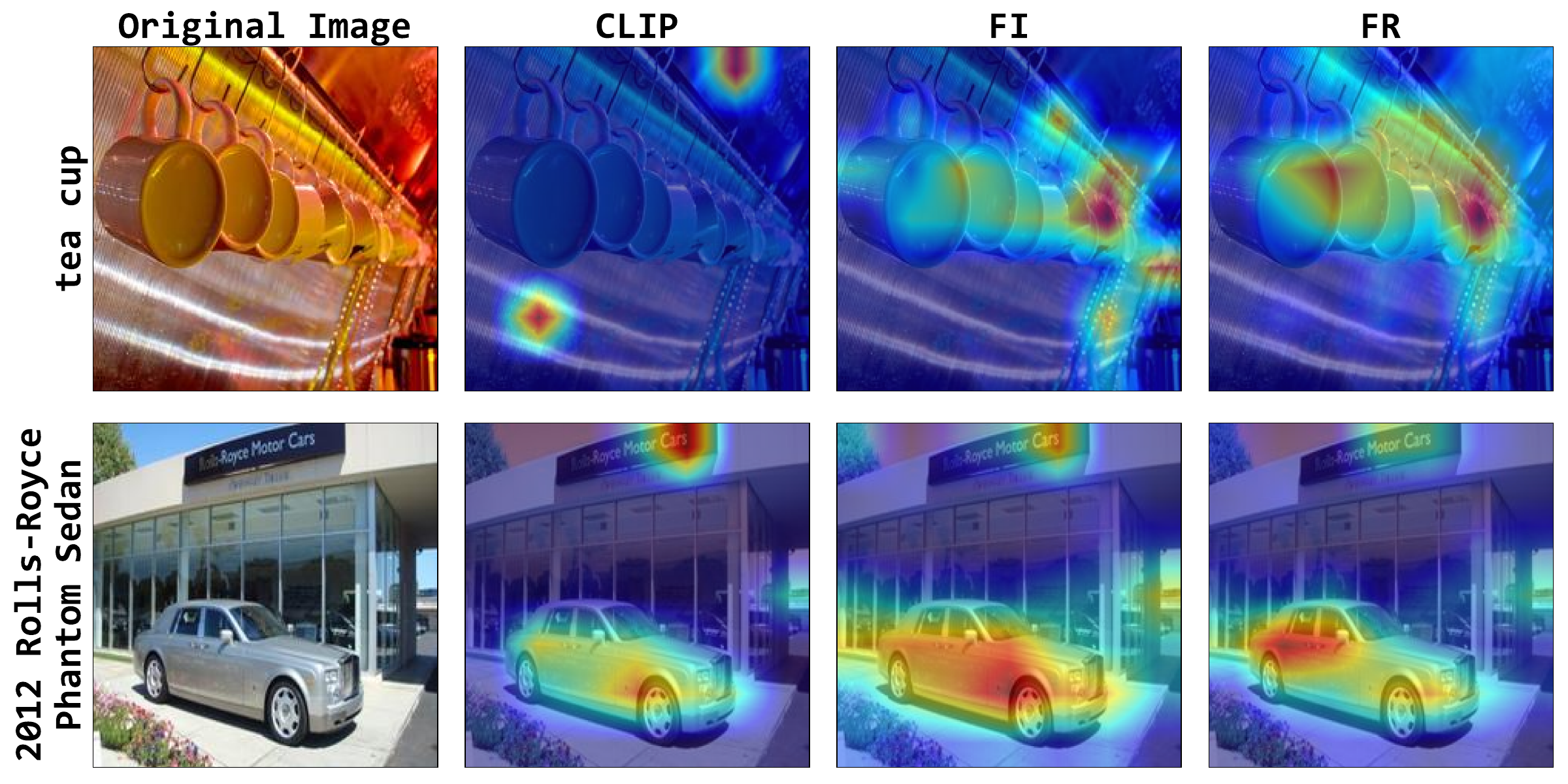}
    \caption{
    GradCAM visualization on exemplar images. Columns show: (left) original images, (center-left to right) GradCAM heatmaps from CLIP visual encoder, FI, and FR respectively. Warmer colors indicate higher attention.
    }
    \label{fig:teaacup}
\end{figure}

\subsection{Sensitivity Study}
Our method contains five hyperparameters, including the logit weights $\alpha$ and $\beta$, the consistency constraint weights $\lambda_1$ and $\lambda_2$, and the consensus constraint weight $\lambda_3$. The sensitivity study (\cref{fig:sensitivity}) of these parameters on ImageNet under the 16-shot setting shows that our method's performance remains stable when each hyperprameter varies within certain range, 
demonstrating its robustness to hyperparameters.

\subsection{Generalization Study}
\noindent\textbf{Domain Adaption}
A domain adaptation study was performed to assess the adaptability of CoMuCo to new domains during inference.
The model was trained on ImageNet with 16-shot samples and evaluated on ImageNet-V2 and ImageNet-Sketch. As presented in \cref{tab:DG}, CoMuCo achieves up to 1.97\% and 1.56\% higher accuracy than the best baseline on the two target datasets, demonstrating solid generalization across domains.

\noindent\textbf{Backbone Generalization}
We evaluate CoMuCo on various visual backbones, including ResNet-50, ResNet-101, ViT/B-32, and ViT/B-16. As shown in \cref{tab:DG}, CoMuCo consistently outperforms all baselines, with 
an average gain of 
1.53\% across all three datasets.
These results confirm its robustness across different architectures.

\subsection{Visualization Analysis}
To further elucidate CoMuCo, a visual analysis of its dual modules was performed. 
Specifically, GradCAM~\cite{selvaraju2017grad} was employed to visualize the model's attention regions when presented with category text and query images. \cref{fig:teaacup} reveals that while the original CLIP model fails to properly focus on the target object, the adapted FR and FI successfully identify it. Moreover, when the CLIP model successfully detected the targets, enhanced comprehensive attention to target objects is achieved through the adapted FR and FI. 
Refer to Appendix D for more results.

%% file: sec/5_conclusion.tex
\section{Conclusion}
In this study, we propose CoMuCo, a Consistency-guided Multi-view Collaborative Optimization framework, for few-shot learning especially in cross-domain scenarios. 
The method employs dual expert modules with prior consistency constraint and multi-view consensus mechanism to enhance learning capacity.
Additionally, we establish a novel cross-domain benchmark for thorough performance assessment across various imaging domains.
Extensive experiments support that CoMuCo substantially boosts model performance.
This study offers a new perspective on efficient transfer learning with vision-language models, and CoMuCo is expected to work well under more scenarios.

%% file: sec/X_suppl.tex
\section{A. Detailed Proofs of Theoretical Claims}
In this appendix, we provide the complete mathematical derivations for the theoretical claims made in the main body of the paper regarding the L1 consistency constraint and the Jeffreys divergence.
\subsection{A.1 Proof of Theorem 1: Equivalence of L1 Regularization and Laplace Prior}
\begin{theorem}
Imposing an $L_1$ regularization on the deviation vector $\bm{\delta}$ is equivalent to assuming an independent Laplace prior in a Maximum a Posteriori (MAP) estimation framework, such that:
\begin{equation}
-\log p(\bm{\delta}) \propto \frac{1}{b} \|\bm{\delta}\|_1.
\end{equation}
\end{theorem}

\begin{proof}
The proof follows directly from the definition of an independent Laplace prior and the properties of logarithms.

Let the deviation vector be $\bm{\delta} = (\delta_1, \delta_2, \dots, \delta_C) \in \mathbb{R}^C$. We start with the assumption stated in Definition 1: each component $\delta_c$ is independently drawn from a zero-mean Laplace distribution with scale parameter $b > 0$. The probability density function (PDF) for a single component is:
\begin{equation}
    p(\delta_c) = \frac{1}{2b} \exp\left(-\frac{|\delta_c|}{b}\right).
\end{equation}
Due to the independence assumption, the joint probability density for the entire vector $\bm{\delta}$ is the product of the individual component densities:
\begin{equation}
    p(\bm{\delta}) = \prod_{c=1}^{C} p(\delta_c) = \prod_{c=1}^{C} \left[ \frac{1}{2b} \exp\left(-\frac{|\delta_c|}{b}\right) \right].
\end{equation}
In a MAP estimation framework, the regularization term corresponds to the negative log-prior, $-\log p(\bm{\delta})$. We now compute this term:
\begin{align}
    -\log p(\bm{\delta}) &= -\log \left( \prod_{c=1}^{C} \left[ \frac{1}{2b} \exp\left(-\frac{|\delta_c|}{b}\right) \right] \right) \nonumber\\
    &= -\sum_{c=1}^{C} \log \left( \frac{1}{2b} \exp\left(-\frac{|\delta_c|}{b}\right) \right) 
    \nonumber\\
    &= -\sum_{c=1}^{C} \left[ \log\left(\frac{1}{2b}\right) - \frac{|\delta_c|}{b} \right] 
    \nonumber\\
    &= -\sum_{c=1}^{C} \left(-\log(2b)\right) + \sum_{c=1}^{C} \frac{|\delta_c|}{b} \nonumber\\
    &= C \log(2b) + \frac{1}{b} \sum_{c=1}^{C} |\delta_c|. \label{eq:proof1_step3}
\end{align}
By definition, the $L_1$-norm of the vector $\bm{\delta}$ is $\|\bm{\delta}\|_1 = \sum_{c=1}^{C} |\delta_c|$. Substituting this into Equation~\ref{eq:proof1_step3}, we get:
\begin{equation}
    -\log p(\bm{\delta}) = \frac{1}{b} \|\bm{\delta}\|_1 + C \log(2b).
\end{equation}
In the context of an optimization objective, the term $C \log(2b)$ is a constant with respect to $\bm{\delta}$ and can be absorbed into the overall loss function or ignored. Therefore, the part of the negative log-prior that depends on $\bm{\delta}$ is directly proportional to its $L_1$-norm:
\begin{equation}
    -\log p(\bm{\delta}) \propto \frac{1}{b} \|\bm{\delta}\|_1.
\end{equation}
This completes the proof. Minimizing a term proportional to $\|\bm{\delta}\|_1$ is thus equivalent to maximizing the log-likelihood of a model that assumes the deviation $\bm{\delta}$ follows an independent Laplace distribution.
\end{proof}

\subsection{A.2 Proof of Theorem 2: Jeffreys Divergence as a Fourth-Order Approximation to Squared Geodesic Distance}

\begin{theorem}
Let $\mathcal{M}$ be a statistical manifold equipped with the Fisher-Rao Riemannian metric. For any two sufficiently close distributions $P$ and $Q$ on $\mathcal{M}$, with local coordinates $\pi_P$ and $\pi_Q$ respectively, the Jeffreys divergence between $P$ and $Q$ satisfies:
\begin{equation}
D_J(P, Q) = d^2(P, Q) + O(\|\pi_Q - \pi_P\|^4),
\end{equation}
where $d^2(P, Q)$ is the squared geodesic distance under the Fisher information metric.
\end{theorem}

\begin{proof}
Let $\pi$ denote the coordinate of $P$, and $\pi + \Delta\pi$ denote the coordinate of $Q$. Let $p(x|\pi)$ and $p(x|\pi + \Delta\pi)$ be the corresponding probability densities.

\paragraph{Step 1: Taylor expansion of $\log p(x|\pi+\Delta\pi)$}
We expand $\log p(x|\pi + \Delta\pi)$ at $\pi$ using multivariate Taylor series:
\begin{align}
\log p(x|\pi + \Delta\pi) =& \log p(x|\pi) \nonumber \\
&+ \sum_i \Delta\pi^i \partial_i \log p \nonumber \\
&+ \frac{1}{2} \sum_{i,j} \Delta\pi^i \Delta\pi^j \partial_i \partial_j \log p \nonumber \\
&+ \frac{1}{6} \sum_{i,j,k} \Delta\pi^i \Delta\pi^j \Delta\pi^k \partial_i \partial_j \partial_k \log p \nonumber \\ 
&+ O(\|\Delta\pi\|^4),
\end{align}
where $\partial_i = \frac{\partial}{\partial \pi^i}$ and all derivatives are evaluated at $\pi$.

\noindent \textbf{Step 2: Taylor expansion of $D_{\mathrm{KL}}(P \| Q)$}
Using the definition of KL divergence:
\begin{equation}
D_{\mathrm{KL}}(P \| Q) = \int p(x|\pi) \left[\log p(x|\pi) - \log p(x|\pi + \Delta\pi)\right] dx,    
\end{equation}
we substitute the Taylor expansion:
\begin{align}
D_{\mathrm{KL}}(P \| Q) =& - \sum_i \Delta\pi^i E_P[\partial_i \log p] \nonumber \\
&- \frac{1}{2} \sum_{i,j} \Delta\pi^i \Delta\pi^j E_P[\partial_i \partial_j \log p] \nonumber \\
&- \frac{1}{6} \sum_{i,j,k} \Delta\pi^i \Delta\pi^j \Delta\pi^k E_P[\partial_i \partial_j \partial_k \log p] \nonumber \\
&+ O(\|\Delta\pi\|^4).
\end{align}

We evaluate each expectation term:

\begin{itemize}
    \item \textbf{First-order term:} By the definition of the score function,
    \begin{equation}
        E_P[\partial_i \log p] = \int p \cdot \frac{\partial_i p}{p} dx = \partial_i \int p(x|\pi) dx = 0.
    \end{equation}
    
    \item \textbf{Second-order term:} The Fisher information matrix is defined by
    \begin{equation}
    g_{ij}(\pi) = E_P[\partial_i \log p \cdot \partial_j \log p] = -E_P[\partial_i \partial_j \log p],
    \end{equation}
    so the second-order contribution becomes:
    \begin{align}
    -\frac{1}{2} \sum_{i,j} \Delta\pi^i \Delta\pi^j E_P[\partial_i \partial_j \log p] \nonumber
    =& \frac{1}{2} \sum_{i,j} g_{ij}(\pi) \Delta\pi^i \Delta\pi^j \nonumber \\
    =& \frac{1}{2} d^2(P, Q).
    \end{align}
    
    \item \textbf{Third-order term:} Define the third-order tensor:
    \begin{equation}
    T_{ijk} = E_P[\partial_i \partial_j \partial_k \log p],
    \end{equation}
    so the third-order contribution is:
    \begin{equation}
    -\frac{1}{6} \sum_{i,j,k} T_{ijk} \Delta\pi^i \Delta\pi^j \Delta\pi^k.
    \end{equation}
\end{itemize}

Thus, we obtain the local expansion:
\begin{align}
\label{eq:KL_forward}
D_{\mathrm{KL}}(P \| Q) =& \frac{1}{2} d^2(P, Q) \nonumber \\
&- \frac{1}{6} \sum_{i,j,k} T_{ijk} \Delta\pi^i \Delta\pi^j \Delta\pi^k \nonumber \\
&+ O(\|\Delta\pi\|^4).
\end{align}

\paragraph{Step 3: Expansion of $D_{\mathrm{KL}}(Q \| P)$}
Similarly, we expand $D_{\mathrm{KL}}(Q \| P)$:
\begin{align}
\label{eq:KL_reverse}
D_{\mathrm{KL}}(Q \| P) =& \frac{1}{2} d^2(Q, P) \nonumber\\
&+ \frac{1}{6} \sum_{i,j,k} T_{ijk} \Delta\pi^i \Delta\pi^j \Delta\pi^k \nonumber \\
&+ O(\|\Delta\pi\|^4),
\end{align}
where the sign of the third-order term is reversed due to the inversion of arguments. Note that $d^2(Q,P) = d^2(P,Q)$ since geodesic distance is symmetric.

\paragraph{Step 4: Jeffreys divergence and cancellation of third-order terms}
The Jeffreys divergence is the sum of forward and reverse KL divergences:
\begin{align}
D_J(P, Q) &= D_{\mathrm{KL}}(P \| Q) + D_{\mathrm{KL}}(Q \| P) \nonumber \\
&= \left[\frac{1}{2} d^2(P, Q) - \frac{1}{6} \sum T_{ijk} \Delta\pi^i \Delta\pi^j \Delta\pi^k \right] \nonumber \\
&\quad + \left[\frac{1}{2} d^2(P, Q) + \frac{1}{6} \sum T_{ijk} \Delta\pi^i \Delta\pi^j \Delta\pi^k \right] \nonumber \\
&\quad + O(\|\Delta\pi\|^4) \nonumber \\
&= d^2(P, Q) + O(\|\Delta\pi\|^4).
\end{align}

\paragraph{Conclusion:}
The third-order terms cancel exactly due to symmetry. Thus, Jeffreys divergence approximates the squared geodesic distance to third-order accuracy, with a fourth-order error in the parameter difference $\Delta\pi$.
\end{proof}

\begin{table*}[!htbp]
    \centering
    \begin{adjustbox}{max width = 1\textwidth}
    \begin{tabular}{lcccc}
        \toprule
        Name & Number of Classes & Size (Train / Val / Test) & Description & Template\\
        \midrule
        Skin40 \cite{skin40} & 40 & 2000 / - / 400 &  Recognition of skin diseases & ``a photo of a [class].'' \\
        TCGA12 \cite{tcga} & 12 & 22468 / - / 5557 & Recognition of pathological tissues & ``a photo of a [class].'' \\
        RFMiD 2.0 \cite{RFMiD} & 15 & 2256 / - / 2256 & Recognition of retinal fundus diseases & ``a fundus image of [class].'' \\
        NWPU-RESISC45 \cite{nwpu} & 45 & 27000 / - / 4500 & Recognition of remote sensing scenes & ``a centered satellite photo of [class].'' \\
        NEU-CLS \cite{neucls} & 6 & 1200 / - / 600 & Recognition of surface defects of the hot-rolled steel & ``a photo of a hot-rolled steel plate with [class].'' \\
        IP102 \cite{IP102} & 102 & 45095 / 7508 / 22619 & Recognition of crop pests and diseases & ``a photo of a [class].'' \\
        Galaxy10 DECaLS \cite{galaxy10} & 10 & 16736 / - / 1000 & Recognition of the morphology of galaxies &``a photo of a [class].'' \\
        \bottomrule
    \end{tabular}
    \end{adjustbox}
    \caption{Summary of 7 datasets for cross-domain few-shot learning.
    }
    \label{tab:dataset_summary_cross}
\end{table*}

\begin{table*}[!t]
    \centering
    \begin{adjustbox}{max width = 1\textwidth}
    \begin{tabular}{lcccc}
        \toprule
        Name & Number of Classes & Size (Train / Val / Test) & Description & Template\\
        \midrule
        ImageNet \cite{ImageNet} & 1000 & 1.28M / - /50000 &  Recognition of generic objects & \multirow{3}{*}{Ensemble of 7 selected templates} \\
        ImageNet-V2 \cite{imagenetV2} & 1000 & - / - / 10000 & New test data for ImageNet & \\
        ImageNet-Sketch \cite{imagenetsketch} & 1000 & - / - / 50889 & Sketch-style images of ImageNet classes & \\
        \midrule
        Caltech101 \cite{Caltech} & 100 & 4128 / 1649 / 2465 & Recognition of generic objects & ``a photo of a [class].'' \\
        OxfordPets \cite{oxfordpets} & 37 & 2944 / 736 / 3669 & Fine-grained classification of pets & ``a photo of a [class], a type of pet.'' \\
        StanfordCars \cite{stanfordcars} & 196 & 6509 / 1635 / 8041 & Fine-grained classification of cars & ``a photo of a [class].'' \\
        Flowers102 \cite{flower102} & 102 & 4093 / 1633 / 2463 & Fine-grained classification of flowers & ``a photo of a [class], a type of flower.'' \\
        Food101 \cite{food101} & 101 & 50500 / 20200 / 30300 & Fine-grained classification of foods & ``a photo of a [class], a type of food.'' \\
        FGVCAircraft \cite{FGVCAircraft} & 100 & 3334 / 3333 / 3333 & Fine-grained classification of aircrafts &``a photo of a [class], a type of aircraft.'' \\
        SUN397 \cite{sun397} & 397 & 15880 / 3970 / 19850 & Scene classification & ``a photo of a [class].'' \\
        DTD \cite{DTD} & 47 & 2820 / 1128 / 1692 & Texture classification & ``[class] texture.'' \\
        EuroSAT \cite{eurosat} & 10 & 13500 / 5400 / 8100 & Land use \& cover classification with satellite images & ``a centered satellite photo of [class].'' \\
        UCF101 \cite{ucf101} & 101 & 7639 / 1898 / 3783 & Action recognition & ``a photo of a person doing [class].'' \\
        \bottomrule
    \end{tabular}
    \end{adjustbox}
    \caption{Summary of 11 datasets for few-shot learning and 2 target datasets of domain generalization. The 7 selected templates \cite{tip-adapter} for ImageNet series datasets are ``itap of a [class].'', ``a bad photo of the [class].'', ``a origami [class].'', ``a photo of the large [class].'', ``a [class] in a video game.'', ``art of the [class].'' and ``a photo of the small [class].''
    }
    \label{tab:dataset_summary}
\end{table*}

\begin{table*}[!ht]
\centering
\begin{adjustbox}{max width = 0.84\textwidth}
\begin{tabularx}{\textwidth}{>{\centering\arraybackslash}m{2cm} | >{\centering\arraybackslash}X}
\toprule
\textbf{Dataset} & \textbf{Example Images} \\
\midrule

Skin40 &
\adjustimage{valign=c, width=0.18\linewidth}{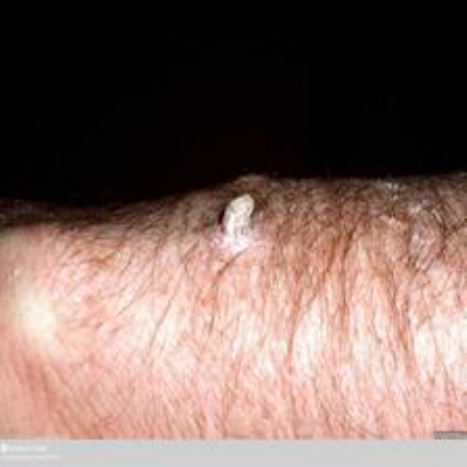} \hfill
\adjustimage{valign=c, width=0.18\linewidth}{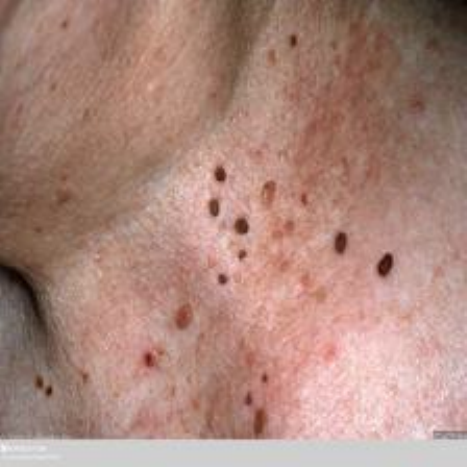} \hfill
\adjustimage{valign=c, width=0.18\linewidth}{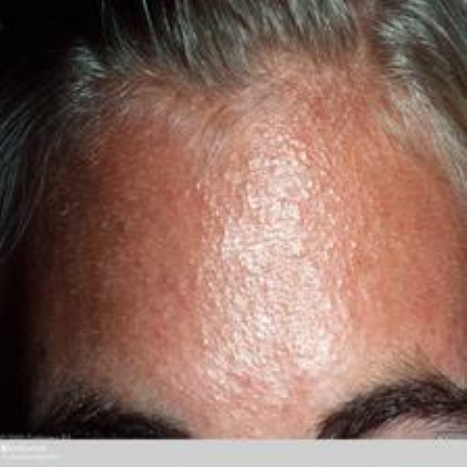} \hfill
\adjustimage{valign=c, width=0.18\linewidth}{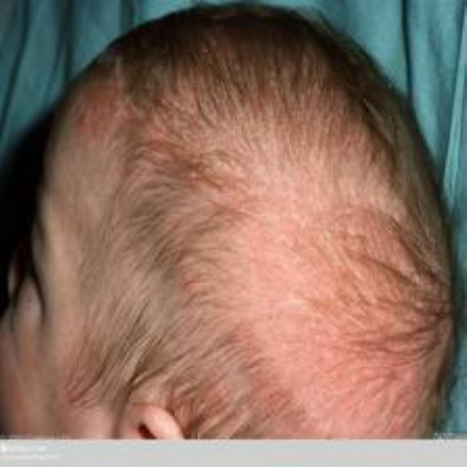} \hfill
\adjustimage{valign=c, width=0.18\linewidth}{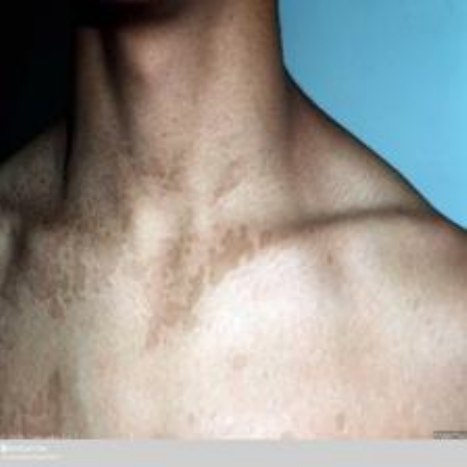} \\
\addlinespace
\midrule

TCGA12 &
\adjustimage{valign=c, width=0.18\linewidth}{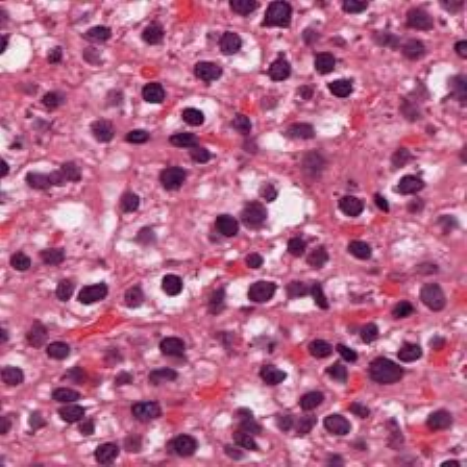} \hfill
\adjustimage{valign=c, width=0.18\linewidth}{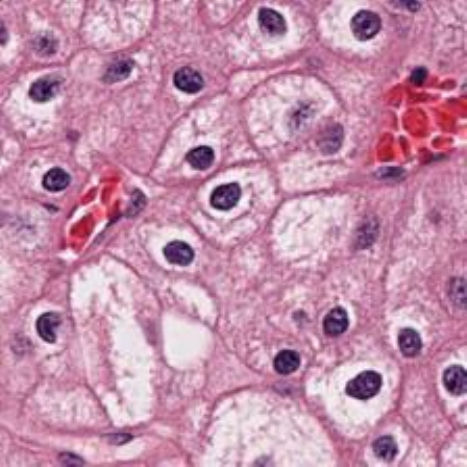} \hfill
\adjustimage{valign=c, width=0.18\linewidth}{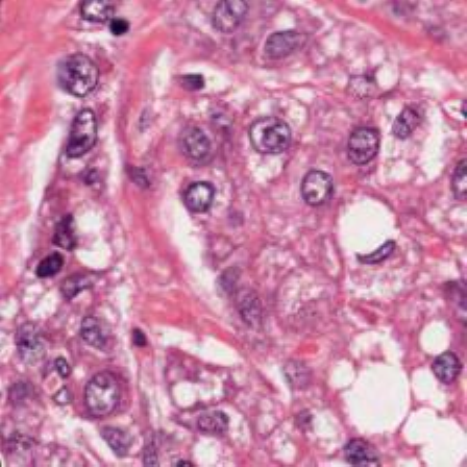} \hfill
\adjustimage{valign=c, width=0.18\linewidth}{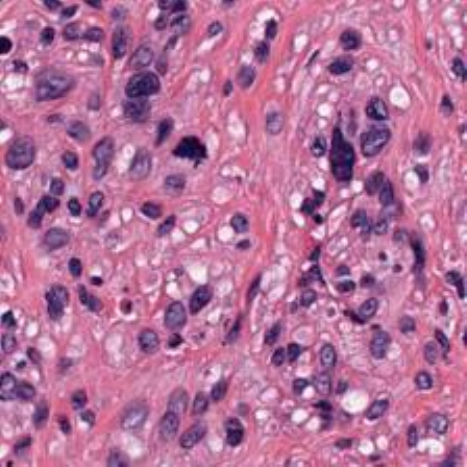} \hfill
\adjustimage{valign=c, width=0.18\linewidth}{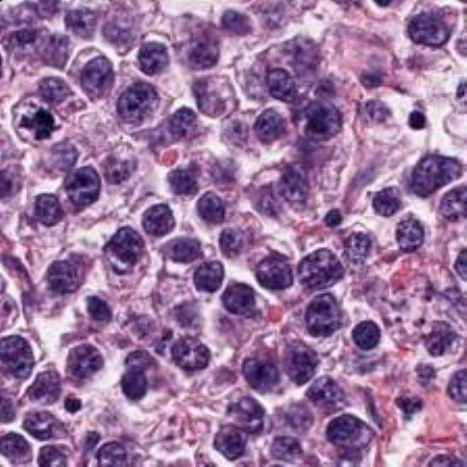} \\
\addlinespace
\midrule

RFMiD12 &
\adjustimage{valign=c, width=0.18\linewidth}{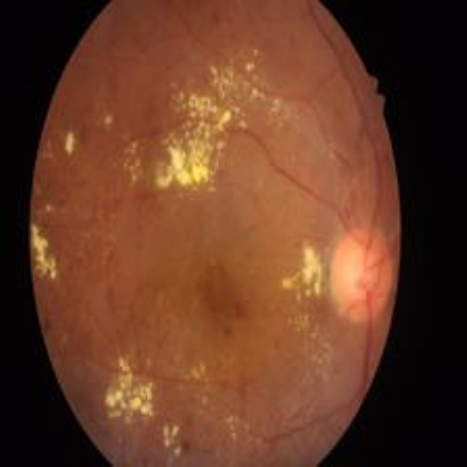} \hfill
\adjustimage{valign=c, width=0.18\linewidth}{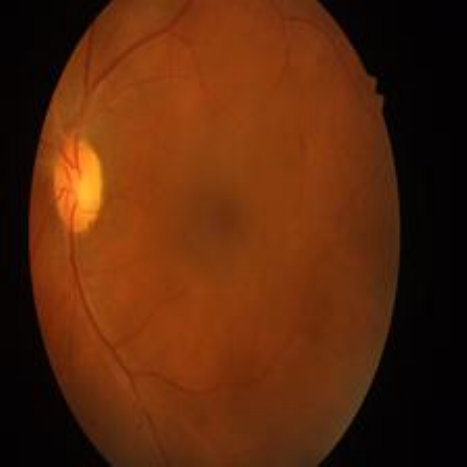} \hfill
\adjustimage{valign=c, width=0.18\linewidth}{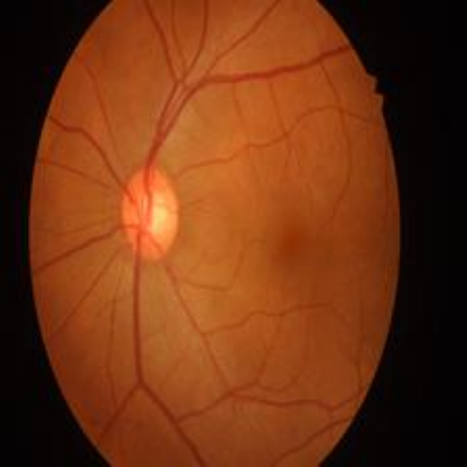} \hfill
\adjustimage{valign=c, width=0.18\linewidth}{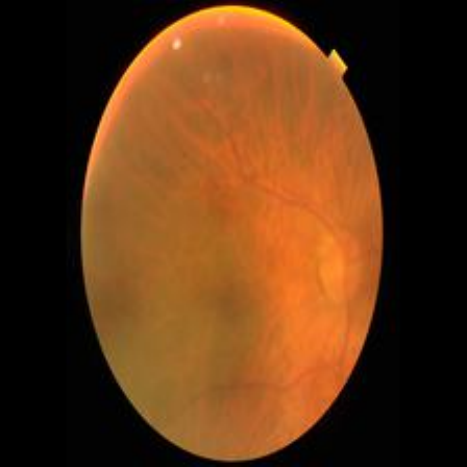} \hfill
\adjustimage{valign=c, width=0.18\linewidth}{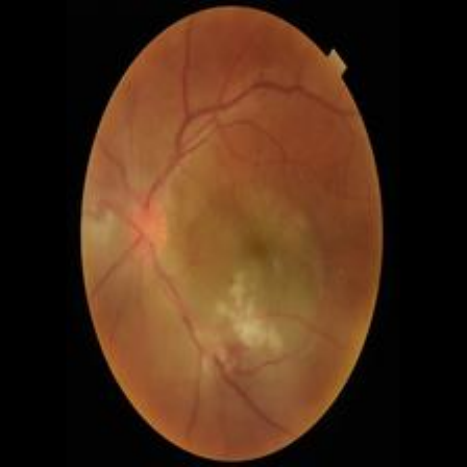} \\
\addlinespace
\midrule

NWPU-RESISC45 &
\adjustimage{valign=c, width=0.18\linewidth}{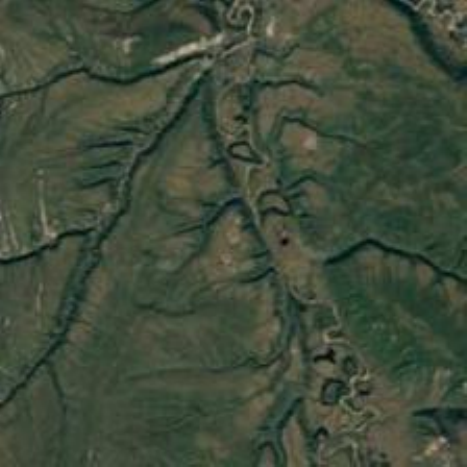} \hfill
\adjustimage{valign=c, width=0.18\linewidth}{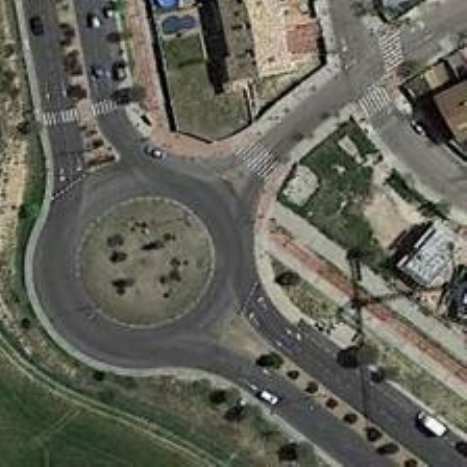} \hfill
\adjustimage{valign=c, width=0.18\linewidth}{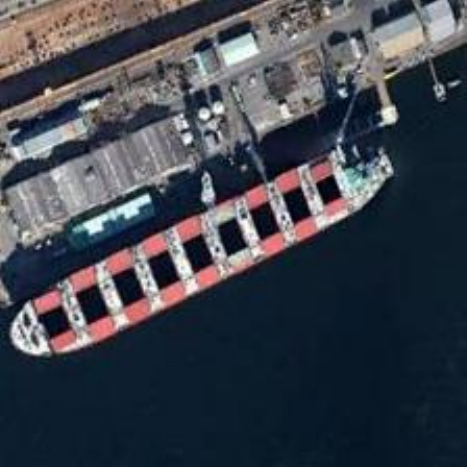} \hfill
\adjustimage{valign=c, width=0.18\linewidth}{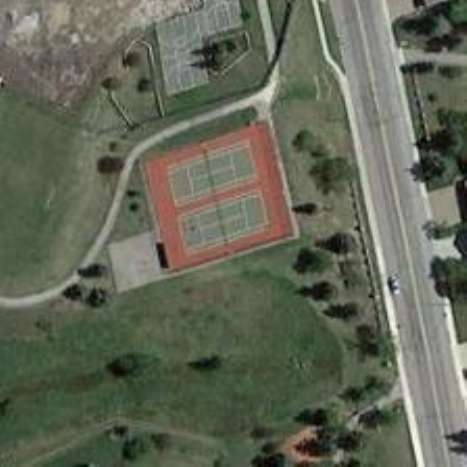} \hfill
\adjustimage{valign=c, width=0.18\linewidth}{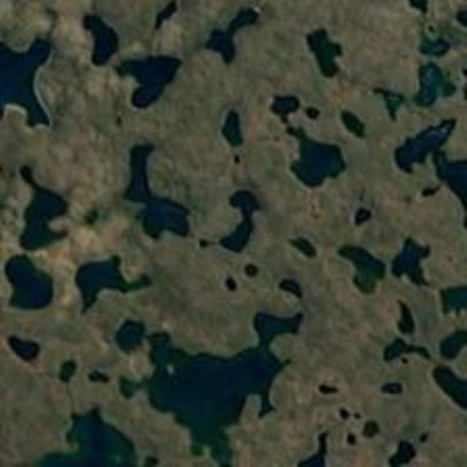} \\
\addlinespace
\midrule

NEU-CLS &
\adjustimage{valign=c, width=0.18\linewidth}{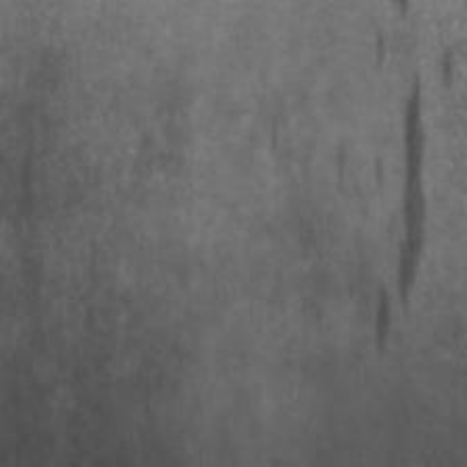} \hfill
\adjustimage{valign=c, width=0.18\linewidth}{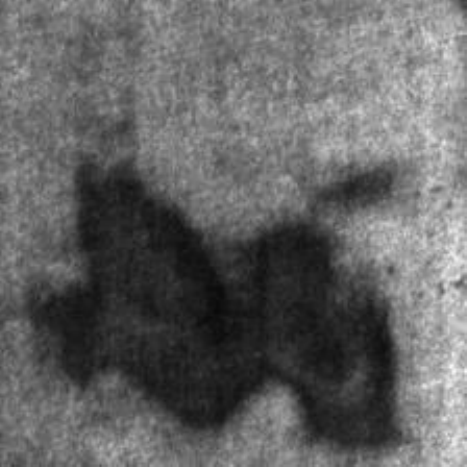} \hfill
\adjustimage{valign=c, width=0.18\linewidth}{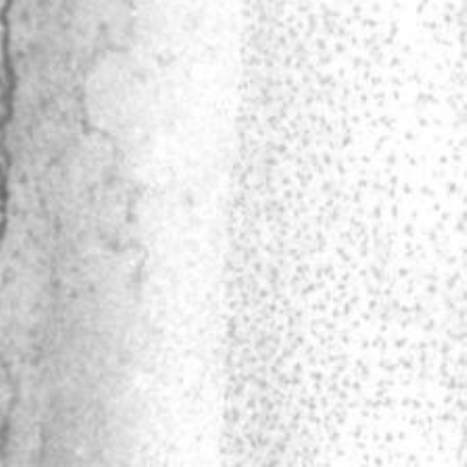} \hfill
\adjustimage{valign=c, width=0.18\linewidth}{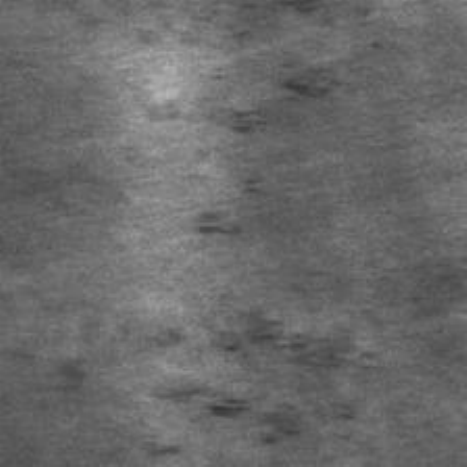} \hfill
\adjustimage{valign=c, width=0.18\linewidth}{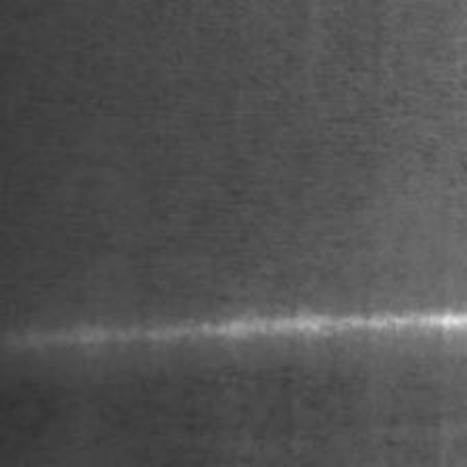} \\
\addlinespace
\midrule

IP102 &
\adjustimage{valign=c, width=0.18\linewidth}{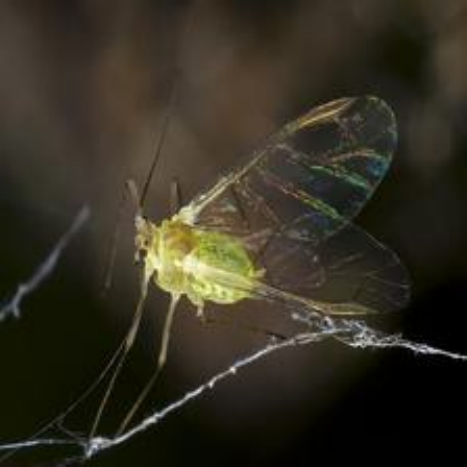} \hfill
\adjustimage{valign=c, width=0.18\linewidth}{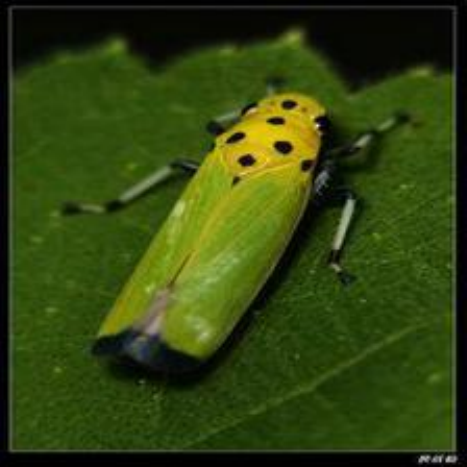} \hfill
\adjustimage{valign=c, width=0.18\linewidth}{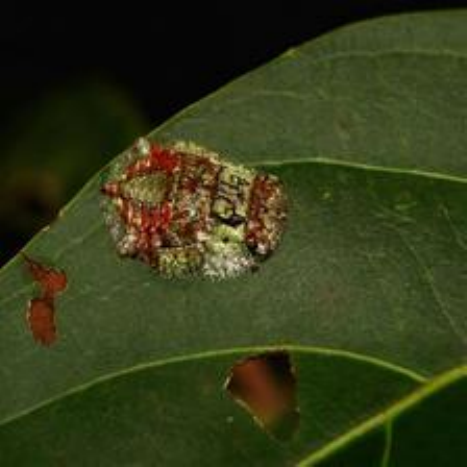} \hfill
\adjustimage{valign=c, width=0.18\linewidth}{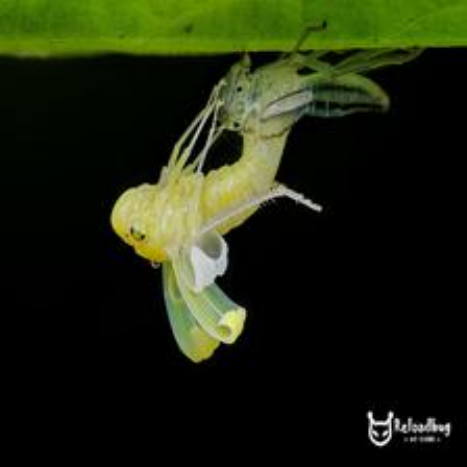} \hfill
\adjustimage{valign=c, width=0.18\linewidth}{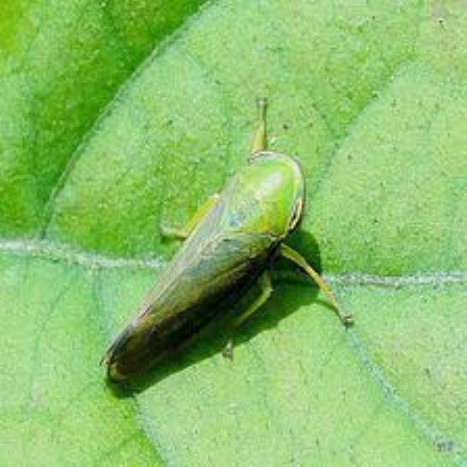} \\
\addlinespace
\midrule

Galaxy 10 DECaLS &
\adjustimage{valign=c, width=0.18\linewidth}{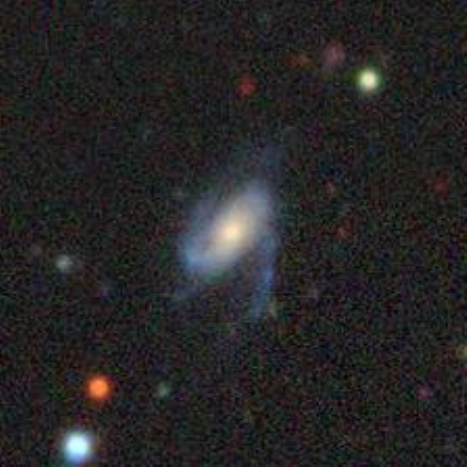} \hfill
\adjustimage{valign=c, width=0.18\linewidth}{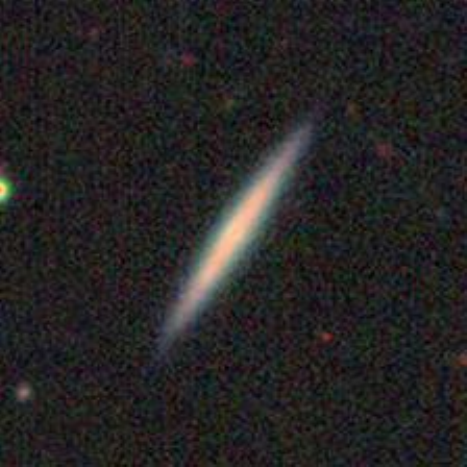} \hfill
\adjustimage{valign=c, width=0.18\linewidth}{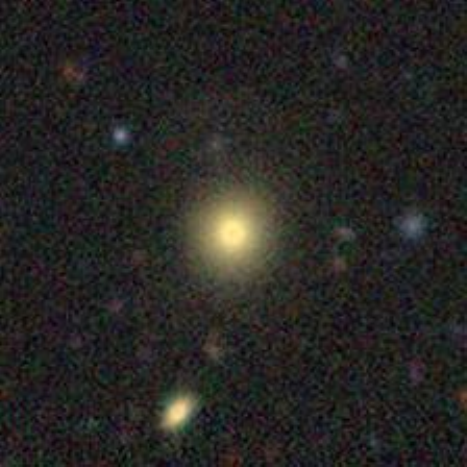} \hfill
\adjustimage{valign=c, width=0.18\linewidth}{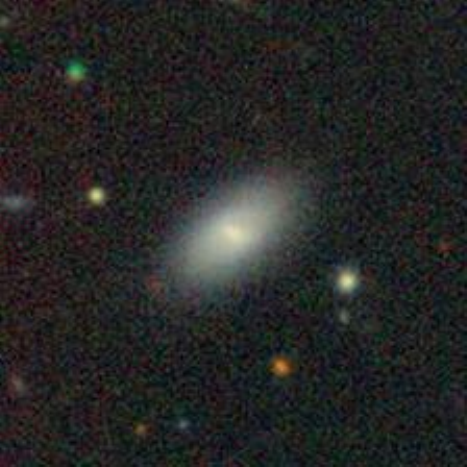} \hfill
\adjustimage{valign=c, width=0.18\linewidth}{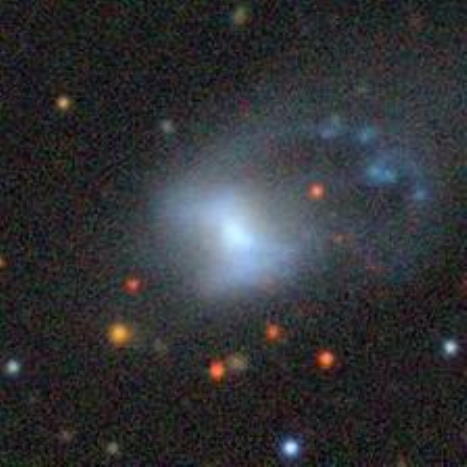} \\

\bottomrule
\end{tabularx}
\end{adjustbox}
\caption{Example images from each dataset used in our benchmark. Each dataset showcases different domain characteristics in terms of semantics and visual style.}
\label{tab:dataset_examples}
\end{table*}

\section{B. Datasets Details}
\label{app:datasets}

\subsection{B.1 Cross-Domain Few-Shot Benchmark}
\label{app:datasets_cross}
The proposed cross-domain few-shot benchmark includes 7 datasets that cover various domains. An overview of these datasets is presented in \cref{tab:dataset_summary_cross}, with example images shown in \cref{tab:dataset_examples}. For the cross-domain few-shot benchmark, we use the following datasets:
\subsubsection{Skin40}
The Skin40~\cite{skin40} dataset comprises 40 categories of dermatological conditions and is a subset of the SD-198~\cite{sd198} dataset. For our research, we utilize the training and testing sets of the Skin40 dataset to conduct training and evaluation in a few-shot learning setting. 

The specific categories included in the dataset are: Ichthyosis, Onychomycosis, Alopecia Areata, Actinic solar Damage (Actinic Keratosis), Actinic solar Damage (Pigmentation), Keratoacanthoma, Perioral Dermatitis, Allergic Contact Dermatitis, Congenital Nevus, Stasis Edema, Pityrosporum Folliculitis, Tinea Faciale, Tinea Corporis, Epidermoid Cyst, Seborrheic Keratosis, Tinea Manus, Compound Nevus, Cutaneous Horn, Nevus Incipiens, Sebaceous Gland Hyperplasia, Dermatofibroma, Eczema, Tinea Pedis, Dyshidrosiform Eczema, Rhinophyma, Psoriasis, Blue Nevus, Acne Vulgaris, Actinic solar Damage (Solar Elastosis), Seborrheic Dermatitis, Malignant Melanoma, Pyogenic Granuloma, Stasis Dermatitis, Steroid Use Abuse/Misuse Dermatitis, Dysplastic Nevus, Basal Cell Carcinoma, Tinea Versicolor, Stasis Ulcer, Skin Tag, and Inverse Psoriasis.

\subsubsection{TCGA12}
The TCGA12 dataset focuses on the recognition of pathological tissue images and is derived from the larger TCGA~\cite{tcga} (The Cancer Genome Atlas) dataset. 

This dataset primarily consists of histopathological images from twelve types of cancer: Thymoma, Sarcoma, Prostate Adenocarcinoma, Glioblastoma Multiforme, Kidney Renal Clear Cell Carcinoma, Thyroid Carcinoma, Uveal Melanoma, Testicular Germ Cell Tumors, Kidney Chromophobe, Adrenocortical Carcinoma, Brain Lower Grade Glioma, and Liver Hepatocellular Carcinoma.

\subsubsection{RFMiD 2.0}
The RFMiD12~\cite{RFMiD} (Retinal Fundus Multi-Disease Image Dataset 2.0) is an ophthalmic dataset consisting of approximately 860 retinal fundus images. Due to the presence of categories with insufficient sample sizes for 16-shot settings (where we require at least 10 test samples per class), our study utilized a subset of the 15 most populous classes for cross-domain benchmarking. 

These specific classes include: diabetic retinopathy, age-related macular degeneration, drusen, myopia, branch retinal vein occlusion, tessellation, laser scar, central serous retinopathy, optic disc cupping, optic disc pallor, optic disc edema, chorioretinitis, macular hole, retinitis, and healthy/unaffected retinal fundus images.

\subsubsection{NWPU-RESISC45}
The NWPU-RESISC45~\cite{nwpu} remote sensing dataset is a large-scale, publicly available dataset designed for scene classification in remote sensing imagery. This dataset encompasses 45 distinct scene categories, with each category containing 700 images. For our study, we randomly selected 100 images from each category to form the test set, while the remaining images were used for training. 

The specific categories included in the dataset are: airplane, airport, baseball diamond, basketball court, beach, bridge, chaparral, church, circular farmland, cloud, commercial area, dense residential, desert, forest, freeway, golf course, ground track field, harbor, industrial area, intersection, island, lake, meadow, medium residential, mobile home park, mountain, overpass, palace, parking lot, railway, railway station, rectangular farmland, river, roundabout, runway, sea ice, ship, snowberg, sparse residential, stadium, storage tank, tennis court, terrace, thermal power station, and wetland.

\subsubsection{NEU-CLS}
The Northeastern University (NEU) surface defect database~\cite{neucls} is a comprehensive dataset that collects six types of typical surface defects found in hot-rolled steel strips. The database consists of 1,800 grayscale images, with 300 samples for each of the six defect categories. For our study, we randomly selected 100 images from each category to form the test set, while the remaining images were used for training. 

These defects include rolled-in scale (RS), patches (Pa), crazing (Cr), pitted surface (PS), inclusion (In), and scratches (Sc).

\begin{table*}[!t]
    \centering
    \begin{adjustbox}{max width = 0.9\textwidth}
    \begin{tabular}{lccccccccc}
        \toprule
        Method & Setting & Skin40 & TCGA12 & RFMiD12 & NWPU-RESISC45 & NEU-CLS & IP102 & Galaxy10 DECaLS & Average \\
        \midrule
        CLIP~\cite{CLIP} & \multirow{7}{*}{1-shot} & 15.20 & 7.60 & 5.70 & 49.20 & 17.30 & 13.60 & 13.90 & 17.50 \\
        CoOp~\cite{CoOp} & & 25.00 & 22.40 & 17.40 & 52.33 & 77.57 & 12.27 & 23.13 & 32.87 \\
        Tip-Adapter-F~\cite{tip-adapter} & & 28.25 & 20.63 & 19.33 & 60.57 & 77.50 & \textbf{19.04} & 26.13 & 35.92 \\
        TaskRes~\cite{taskres} & & \textbf{29.43} & 22.50 & 21.43 & 60.43 & 77.23 & 17.80 & 25.00 & 36.26 \\
        LP++~\cite{huang2024lp++} & & 24.05 & 21.08 & 12.86 & 58.64 & 63.02 & 17.23 & 16.90 & 30.54 \\
        TCP~\cite{yao2024tcp} & & 29.40 & \textbf{24.57} & 22.63 & 61.80 & 80.37 & 17.60 & 24.67 & 37.29 \\
        \rowcolor{Gray}
        CoMuCo (Ours) & & 27.57 & 18.23 & \textbf{25.67} & \textbf{65.33} & \textbf{81.83} & 18.43 & \textbf{30.27} & \textbf{38.19} \\
        
        \midrule
        CLIP~\cite{CLIP} & \multirow{7}{*}{2-shot} & 15.20 & 7.60 & 5.70 & 49.20 & 17.30 & 13.60 & 13.90 & 17.50 \\
        CoOp~\cite{CoOp} & & 31.43 & 27.03 & 25.77 & 58.40 & 84.50 & 15.77 & 29.30 & 38.89 \\
        Tip-Adapter-F~\cite{tip-adapter} & & 32.50 & 30.12 & 32.00 & 63.81 & 87.28 & 20.86 & 29.07 & 42.23 \\
        TaskRes~\cite{taskres} & & 34.17 & 31.57 & 30.73 & 66.50 & 85.77 & \textbf{20.93} & 28.27 & 42.56 \\
        LP++~\cite{huang2024lp++} & & 33.40 & 24.33 & 21.09 & 63.26 & 85.08 & 19.69 & 21.38 & 38.32 \\
        TCP~\cite{yao2024tcp} & & 33.80 & \textbf{31.63} & 32.00 & 66.80 & \textbf{88.23} & 20.40 & 28.27 & 43.02 \\
        \rowcolor{Gray}
        CoMuCo (Ours) & & \textbf{36.10} & 29.70 & \textbf{36.57} & \textbf{69.97} & 86.77 & 20.53 & \textbf{37.23} & \textbf{45.27} \\
        
        \midrule
        CLIP~\cite{CLIP} & \multirow{7}{*}{4-shot} & 15.20 & 7.60 & 5.70 & 49.20 & 17.30 & 13.60 & 13.90 & 17.50 \\
        CoOp~\cite{CoOp} & & 40.73 & 40.00 & 25.37 & 64.93 & 91.03 & 22.97 & 30.47 & 45.07 \\
        Tip-Adapter-F~\cite{tip-adapter} & & 41.67 & \textbf{43.14} & 32.33 & 69.21 & 94.00 & 24.90 & 31.93 & 48.17 \\
        TaskRes~\cite{taskres} & & 42.60 & 40.43 & 32.63 & 71.10 & 91.90 & 24.83 & 30.13 & 47.66 \\
        LP++~\cite{huang2024lp++} & & 39.78 & 29.47 & 31.25 & 68.44 & 88.62 & 21.11 & 27.48 & 43.74 \\
        TCP~\cite{yao2024tcp} & & \textbf{45.07} & 38.17 & 30.93 & 71.23 & 92.37 & 25.87 & 28.93 & 47.51 \\
        \rowcolor{Gray}
        CoMuCo (Ours) & & 43.67 & 38.77 & \textbf{39.27} & \textbf{75.57} & \textbf{94.17} & \textbf{27.00} & \textbf{38.20} & \textbf{50.95} \\
        
        \midrule
        CLIP~\cite{CLIP} & \multirow{7}{*}{8-shot} & 15.20 & 7.60 & 5.70 & 49.20 & 17.30 & 13.60 & 13.90 & 17.50 \\
        CoOp~\cite{CoOp} & & 47.33 & 48.40 & 44.47 & 72.77 & 95.67 & 28.37 & 39.97 & 53.85 \\
        Tip-Adapter-F~\cite{tip-adapter} & & 46.58 & 50.28 & 48.67 & 74.70 & 96.44 & 30.77 & 40.90 & 55.48 \\
        TaskRes~\cite{taskres} & & 47.23 & 44.83 & 47.43 & 76.53 & 94.93 & 29.73 & 36.77 & 53.92 \\
        LP++~\cite{huang2024lp++} & & 43.88 & 38.90 & 36.14 & 73.00 & 92.18 & 30.96 & 33.55 & 49.80 \\
        TCP~\cite{yao2024tcp} & & 48.70 & 48.70 & 50.07 & 77.20 & 95.77 & 31.07 & 35.27 & 55.25 \\
        \rowcolor{Gray}
        CoMuCo (Ours) & & \textbf{53.47} & \textbf{55.03} & \textbf{57.43} & \textbf{81.23} & \textbf{97.60} & \textbf{31.87} & \textbf{48.63} & \textbf{60.75} \\
        
        \midrule
        CLIP~\cite{CLIP} & \multirow{7}{*}{16-shot} & 15.20 & 7.60 & 5.70 & 49.20 & 17.30 & 13.60 & 13.90 & 17.50 \\
        CoOp~\cite{CoOp} & & 53.23 & 54.90 & 46.07 & 78.03 & 96.60 & 33.70 & 44.33 & 58.12 \\
        Tip-Adapter-F~\cite{tip-adapter} & & 54.17 & 56.74 & 48.67 & 79.40 & 96.89 & 35.23 & 42.97 & 59.15 \\
        TaskRes~\cite{taskres} & & 52.83 & 52.83 & 46.27 & 79.83 & 95.83 & 34.03 & 41.07 & 57.53 \\
        LP++~\cite{huang2024lp++} & & 54.10 & 44.91 & 45.87 & 76.88 & 95.78 & 35.64 & 37.55 & 55.82 \\
        TCP~\cite{yao2024tcp} & & 54.73 & 54.73 & 47.90 & 80.30 & 96.60 & 36.10 & 40.60 & 58.71 \\
        \rowcolor{Gray}
        CoMuCo (Ours) & & \textbf{61.93} & \textbf{64.33} & \textbf{56.37} & \textbf{85.40} & \textbf{98.17} & \textbf{40.27} & \textbf{56.83} & \textbf{66.19} \\
        \bottomrule
    \end{tabular}
    \end{adjustbox}
    \caption{Performance comparison on cross-domain few-shot benchmark on ResNet50.
    }
    \label{tab:numerical_compare_cross}
\end{table*}

\subsubsection{IP102}
The IP102~\cite{IP102} dataset is a comprehensive collection that contains over 75,000 images spanning 102 categories. This dataset exhibits a natural long-tailed distribution, reflecting the varying frequencies of occurrence across different categories. Additionally, 19,000 images within the dataset have been annotated with bounding boxes to facilitate object detection tasks. The IP102 dataset features a hierarchical taxonomy, where insect pests that primarily affect a specific agricultural product are grouped together under the same upper-level category. For our study, we utilized the training, validation, and test sets provided by the IP102 dataset for our respective training, validation, and testing phases. 

The specific categories included in the dataset are: rice leaf roller, rice leaf caterpillar, paddy stem maggot, asiatic rice borer, yellow rice borer, rice gall midge, Rice Stemfly, brown plant hopper, white backed plant hopper, small brown plant hopper, rice water weevil, rice leafhopper, grain spreader thrips, rice shell pest, grub, mole cricket, wireworm, white margined moth, black cutworm, large cutworm, yellow cutworm, red spider, corn borer, army worm, aphids, Potosiabre vitarsis, peach borer, english grain aphid, green bug, bird cherry-oat aphid, wheat blossom midge, penthaleus major, longlegged spider mite, wheat phloeothrips, wheat sawfly, cerodonta denticornis, beet fly, flea beetle, cabbage army worm, beet army worm, Beet spot flies, meadow moth, beet weevil, sericaorient alismots chulsky, alfalfa weevil, flax budworm, alfalfa plant bug, tarnished plant bug, Locustoidea, lytta polita, legume blister beetle, blister beetle, therioaphis maculata Buckton, odontothrips loti, Thrips, alfalfa seed chalcid, Pieris canidia, Apolygus lucorum, Limacodidae, Viteus vitifoliae, Colomerus vitis, Brevipoalpus lewisi McGregor, oides decempunctata, Polyphagotars onemus latus, Pseudococcus comstocki Kuwana, parathrene regalis, Ampelophaga, Lycorma delicatula, Xylotrechus, Cicadella viridis, Miridae, Trialeurodes vaporariorum, Erythroneura apicalis, Papilio xuthus, Panonchus citri McGregor, Phyllocoptes oleiverus ashmead, Icerya purchasi Maskell, Unaspis yanonensis, Ceroplastes rubens, Chrysomphalus aonidum, Parlatoria zizyphus Lucus, Nipaecoccus vastalor, Aleurocanthus spiniferus, Tetradacus c Bactrocera minax, Dacus dorsalis (Hendel), Bactrocera tsuneonis, Prodenia litura, Adristyrannus, Phyllocnistis citrella Stainton, Toxoptera citricidus, Toxoptera aurantii, Aphis citricola Vander Goot, Scirtothrips dorsalis Hood, Dasineura sp, Lawana imitata Melichar, Salurnis marginella Guerr, Deporaus marginatus Pascoe, Chlumetia transversa, Mango flat beak leafhopper, Rhytidodera bowrinii white, Sternochetus frigidus, Cicadellidae.

\begin{table*}[h]
    \centering
    \begin{adjustbox}{max width = 0.9\textwidth}
    \begin{tabular}{lccccccccc}
        \toprule
        Method & Setting & Skin40 & TCGA12 & RFMiD12 & NWPU-RESISC45 & NEU-CLS & IP102 & Galaxy10 DECaLS & Average \\
        \midrule
        CLIP~\cite{CLIP} & \multirow{6}{*}{1-shot} & 18.20 & 17.80 & 11.00 & 60.30 & 17.20 & 13.00 & 9.10 & 20.94 \\
        MaPLe~\cite{maple} & & 27.57 & 29.27 & 22.40 & 63.27 & 78.97 & 13.97 & 30.67 & 38.01 \\
        TCP~\cite{yao2024tcp} & & 34.33 & 27.97 & 23.27 & 74.77 & 82.23 & \textbf{22.37} & 31.10 & 42.29 \\
        DePT~\cite{zhang2024dept} & & 33.67 & 29.57 & 24.20 & 68.27 & \textbf{87.33} & 18.87 & 31.40 & 41.90 \\
        TextRefiner~\cite{xie2024textrefiner} & & 27.00 & 26.70 & 18.60 & 63.47 & 77.37 & 17.77 & 28.93 & 37.12 \\
        SkipT~\cite{wu2025skip} & & 30.30 & 29.37 & 23.90 & 72.67 & 72.90 & 17.13 & 31.60 & 39.70 \\
        \rowcolor{Gray}
        CoMuCo (Ours) & & \textbf{34.37} & \textbf{33.83} & \textbf{27.57} & \textbf{74.97} & 87.10 & 20.90 & \textbf{34.13} & \textbf{44.55} \\
        
        \midrule
        CLIP~\cite{CLIP} & \multirow{6}{*}{2-shot} & 18.20 & 17.80 & 11.00 & 60.30 & 17.20 & 13.00 & 9.10 & 20.94 \\
        MaPLe~\cite{maple} & & 36.50 & 39.63 & 35.20 & 66.10 & 85.90 & 17.40 & 41.20 & 45.99 \\
        TCP~\cite{yao2024tcp} & & 42.10 & 36.53 & 32.03 & 77.83 & 86.77 & \textbf{24.97} & 36.93 & 48.17 \\
        DePT~\cite{zhang2024dept} & & 39.07 & 40.33 & 35.10 & 71.97 & 90.60 & 21.63 & 40.40 & 48.44 \\
        TextRefiner~\cite{xie2024textrefiner} & & 35.83 & 32.27 & 30.20 & 69.93 & 85.93 & 20.43 & 36.87 & 44.50 \\
        SkipT~\cite{wu2025skip} & & 39.40 & 38.17 & 32.57 & 75.00 & 85.77 & 20.70 & 39.77 & 47.34 \\
        \rowcolor{Gray}
        CoMuCo (Ours) & & \textbf{42.40} & \textbf{43.20} & \textbf{37.77} & \textbf{79.80} & \textbf{91.67} & 23.43 & \textbf{41.77} & \textbf{51.16} \\
        
        \midrule
        CLIP~\cite{CLIP} & \multirow{6}{*}{4-shot} & 18.20 & 17.80 & 11.00 & 60.30 & 17.20 & 13.00 & 9.10 & 20.94 \\
        MaPLe~\cite{maple} & & 45.83 & 50.17 & 40.20 & 76.10 & 94.07 & 24.83 & 42.00 & 53.31 \\
        TCP~\cite{yao2024tcp} & & 50.13 & 45.33 & 35.60 & 82.20 & 89.23 & \textbf{31.63} & 40.30 & 53.49 \\
        DePT~\cite{zhang2024dept} & & 47.67 & 53.00 & 39.63 & 78.30 & \textbf{96.67} & 27.67 & 45.77 & 55.53 \\
        TextRefiner~\cite{xie2024textrefiner} & & 47.17 & 41.70 & 25.63 & 77.60 & 88.57 & 27.63 & 38.17 & 49.50 \\
        SkipT~\cite{wu2025skip} & & 50.90 & 52.87 & 39.83 & 81.07 & 91.20 & 28.30 & 44.27 & 55.49 \\
        \rowcolor{Gray}
        CoMuCo (Ours) & & \textbf{52.13} & \textbf{54.53} & \textbf{43.73} & \textbf{83.20} & 96.23 & 31.47 & \textbf{45.93} & \textbf{57.96} \\
        
        \midrule
        CLIP~\cite{CLIP} & \multirow{6}{*}{8-shot} & 18.20 & 17.80 & 11.00 & 60.30 & 17.20 & 13.00 & 9.10 & 20.94 \\
        MaPLe~\cite{maple} & & 54.93 & 60.37 & 51.43 & 80.33 & 98.80 & 30.73 & 52.73 & 61.33 \\
        TCP~\cite{yao2024tcp} & & 57.27 & 51.60 & 45.90 & 85.37 & 95.03 & 37.33 & 47.63 & 60.02 \\
        DePT~\cite{zhang2024dept} & & 57.10 & 61.17 & 55.17 & 82.90 & 98.57 & 33.57 & \textbf{54.53} & 63.29 \\
        TextRefiner~\cite{xie2024textrefiner} & & 52.00 & 49.83 & 46.45 & 82.47 & 94.67 & 33.67 & 44.67 & 57.68 \\
        SkipT~\cite{wu2025skip} & & 58.77 & 60.73 & 56.27 & 85.23 & 97.77 & 34.63 & 54.93 & 64.05 \\
        \rowcolor{Gray}
        CoMuCo (Ours) & & \textbf{61.93} & \textbf{62.37} & \textbf{58.03} & \textbf{87.13} & \textbf{99.10} & \textbf{37.40} & 54.33 & \textbf{65.76} \\
        
        \midrule
        CLIP~\cite{CLIP} & \multirow{6}{*}{16-shot} & 18.20 & 17.80 & 11.00 & 60.30 & 17.20 & 13.00 & 9.10 & 20.94 \\
        MaPLe~\cite{maple} & & 63.83 & 65.53 & 55.10 & 85.00 & 98.80 & 37.50 & 59.30 & 66.44 \\
        TCP~\cite{yao2024tcp} & & 62.70 & 55.10 & 47.30 & 87.87 & 95.67 & 42.90 & 50.10 & 63.09 \\
        DePT~\cite{zhang2024dept} & & 66.57 & 66.27 & 58.00 & 86.93 & 99.17 & 41.10 & 60.80 & 68.40 \\
        TextRefiner~\cite{xie2024textrefiner} & & 60.00 & 54.00 & 46.50 & 85.90 & 95.33 & 39.40 & 50.37 & 61.64 \\
        SkipT~\cite{wu2025skip} & & 68.00 & 66.70 & 57.33 & 88.33 & 98.30 & 42.77 & 59.30 & 68.68 \\
        \rowcolor{Gray}
        CoMuCo (Ours) & & \textbf{68.90} & \textbf{68.33} & \textbf{59.63} & \textbf{90.10} & \textbf{99.40} & \textbf{45.80} & \textbf{60.83} & \textbf{70.43} \\
        \bottomrule
    \end{tabular}
    \end{adjustbox}
    \caption{Performance comparison on cross-domain few-shot benchmark on ViT-B/16.
    }
    \label{tab:numerical_compare_cross_vit}
\end{table*}

\subsubsection{Galaxy10 DECaLS}
The Galaxy10 DECaLS~\cite{galaxy10} dataset is a comprehensive collection that contains 17,736 color images of galaxies, each with a resolution of 256x256 pixels. These images are captured in the g, r, and z bands and are categorized into 10 distinct classes. For our study, we randomly selected 100 images from each class to form the test set, while the remaining images were used for training. The specific categories included in the dataset are: Disturbed Galaxies, Merging Galaxies, Round Smooth Galaxies, In-between Round Smooth Galaxies, Cigar Shaped Smooth Galaxies, Barred Spiral Galaxies, Unbarred Tight Spiral Galaxies, Unbarred Loose Spiral Galaxies, Edge-on Galaxies without Bulge, and Edge-on Galaxies with Bulge.

\subsection{B.2 CLIP Benchmark}
\label{app:datasets_clip}
In the main text, our method was assessed on the widely adopted CLIP Benchmark, in alignment with previous work~\cite{CoOp,tip-adapter,taskres}. The benchmark comprises 11 diverse datasets, including ImageNet~\cite{ImageNet}, Caltech101~\cite{Caltech}, Oxford Pets~\cite{oxfordpets}, Stanford Cars~\cite{stanfordcars}, Flowers102~\cite{flower102}, Food101~\cite{food101}, FGVCAircraft~\cite{FGVCAircraft}, SUN397~\cite{sun397}, DTD~\cite{DTD}, EuroSAT~\cite{eurosat}, and UCF101~\cite{ucf101}. These datasets span a broad range of image classification scenarios, encompassing general object recognition, fine-grained object recognition, scene recognition, texture recognition, and satellite imagery analysis. To ensure consistency with previous work~\cite{CoOp, tip-adapter, taskres}, the ``BACKGROUND Google" and ``Faces easy" classes were excluded from the Caltech101 dataset. 
Additionally, robustness under domain shift was analyzed using two ImageNet variants: ImageNet-V2~\cite{imagenetV2}, containing 200 overlapping classes, and ImageNet-Sketch~\cite{imagenetsketch}, encompassing 1,000 classes identical to ImageNet. Consistent with earlier works, ImageNet was used as the source dataset, while the two variants served as target datasets.
An overview of these datasets is presented in \cref{tab:dataset_summary}.

\begin{table}[ht]
\centering
\begin{adjustbox}{max width = 0.45\textwidth}
\begin{tabular}{lccc}
\toprule
\textbf{Type} & \textbf{Tier 1 (Low Shift)} & \textbf{Tier 2 (Medium Shift)} & \textbf{Tier 3 (High Shift)} \\
\midrule
Visually similar but semantically shifted &
NWPU-RESISC45 &
IP102 &
Galaxy 10 DECaLS \\
\addlinespace
Visually and semantically different &
-- &
Skin40 &
TCGA12, RFMiD12, NEU-CLS \\
\bottomrule
\end{tabular}
\end{adjustbox}
\caption{Datasets grouped by domain type and domain shift severity. Tier 1–3 correspond to increasing levels of domain shift, based on CLIP zero-shot performance. Two domain types are considered: (1) datasets visually similar to natural images but with semantic shift, and (2) datasets with significant differences in both texture and semantics.}
\label{tab:benchmark_categorization}
\end{table}

\subsection{B.3 Extended Stratification of the Cross-Domain Few-Shot Benchmark}
To further categorize the datasets, we divide them based on their semantic properties and the zero-shot classification performance of the CLIP model. Specifically, we classify datasets into different levels of domain shift and types of domain difference. The domain shift levels are divided into three tiers according to the zero-shot classification accuracy of the CLIP model. Meanwhile, the domain types are categorized into two groups: one consists of datasets whose images resemble natural images but have task-specific label semantics (e.g., classification tasks requiring expert knowledge), and the other includes datasets whose image styles significantly differ from natural images. The detailed categorization of domain shift levels and domain types is shown in \cref{tab:benchmark_categorization}, respectively.

\begin{table*}[t]
    \centering
    \begin{adjustbox}{max width = 0.9\textwidth}
    \begin{tabular}{lccccccccccccc}
        \toprule
        Method & Setting & ImageNet & Caltech101 & OxfordPets & StanfordCars & Flowers102 & Food101 & FGVCAircraft & SUN397 & DTD & EuroSAT & UCF101 & Average \\
        \midrule
        Zero-Shot CLIP & \multirow{7}{*}{1-shot} & 58.18 & 86.29 & 85.77 & 55.61 & 66.14 & \textbf{77.31} & 17.28 & 58.52 & 42.32 & 37.56 & 61.46 & 58.77  \\
        CoOp & & 57.15 & 87.53 & 85.89 & 55.59 & 68.12 & 74.32 & 9.64 & 60.29 & 44.39 & 50.63 & 61.92 & 59.59 \\
        Tip-Adapter-F & & 60.88 & 88.80 & 86.04 & 56.78 & \textbf{81.17} & 76.22 & 19.01 & 61.23 & \textbf{50.49} & 50.34 & 66.19 & 63.38 \\
        TaskRes & & 61.43 & 88.80 & 83.50 & 58.77 & 78.77 & 74.03 & 21.20 & 61.93 & 50.17 & 61.27 & 64.57 & 64.04 \\
        LP++ & & 61.18 & 88.56 & 84.24 & 57.20 & 78.21 & 76.61 & 19.69 & 62.47 & 46.97 & 57.23 & 65.41 & 63.43 \\
        TCP & & 57.13 & 88.60 & \textbf{87.20} & 55.40 & 77.77 & 75.9 & 19.8 & 61.93 & 47.37 & 61.17 & 65.57 & 63.44 \\
        \rowcolor{Gray}
        CoMuCo (Ours) & & \textbf{61.47} & \textbf{89.37} & 87.10 & \textbf{59.90} & 79.60 & 77.00 & \textbf{22.1} & \textbf{63.33} & 50.27 & \textbf{63.53} & \textbf{67.87} & \textbf{65.59} \\
        \midrule
        Zero-Shot CLIP & \multirow{7}{*}{2-shot} & 58.18 & 86.29 & 85.77 & 55.61 & 66.14 & {77.31} & 17.28 & 58.52 & 42.32 & 37.56 & 61.46 & 58.77  \\
        CoOp & & 57.81 & 87.93 & 82.64 & 58.28 & 77.51 & 72.49 & 18.68 & 59.48 & 45.15 & 61.50 & 64.09 & 62.32 \\
        Tip-Adapter-F & & 61.57 & 89.61 & 86.06 & 61.13 & 85.40 & 77.05 & 21.76 & 63.19 & \textbf{55.32} & 64.76 & 68.99 & 66.80 \\
        TaskRes & & \textbf{62.17} & 90.13 & 84.43 &	62.77 &	\textbf{85.63} & 75.30 &	23.07 &	64.33 &	54.53 &	65.77 &	69.10 &	67.02 \\
        LP++ & & 61.56 & 89.53 & 85.74 & 59.95 & 84.69 & 77.22 & 21.58 & 64.65 & 52.44 & 61.65 & 69.20 & 66.20 \\
        TCP & & 58.77 & 89.53 & 86.83 & 59.67 & 84.47 & 75.87 & 21.93 & 64.20 & 52.93 & 66.73 & 70.27 & 66.47 \\
        \rowcolor{Gray}
        CoMuCo (Ours) & & 61.87 & \textbf{90.67} & \textbf{87.67} & \textbf{63.53} & 84.97 & \textbf{77.47} & \textbf{26.77} & \textbf{65.33} & 54.43 & \textbf{69.97} & \textbf{71.67} & \textbf{68.58} \\
        \midrule
        Zero-Shot CLIP & \multirow{7}{*}{4-shot} & 58.18 & 86.29 & 85.77 & 55.61 & 66.14 & 77.31 & 17.28 & 58.52 & 42.32 & 37.56 & 61.46 & 58.77 \\
        CoOp & & 59.99 & 89.55 & 86.70 & 62.62 & 86.20 & 73.33 & 21.87 & 63.47 & 53.49 & 70.18 & 67.03 & 66.77 \\
        Tip-Adapter-F & & 62.62 & 90.87 & 86.46 & 64.86 & 89.53 & 77.46 & {26.39} & 65.88 & 60.25 & 69.66 & {72.71} & 69.70 \\
        TaskRes & & 62.93 & 90.63 & 86.27 & 66.50 & 89.50 & 76.23 & 24.83 & 66.67 & 59.50 & 72.97 & 69.70 & 69.61 \\
        LP++ & & 62.55 & 90.87 & 86.94 & 63.44 & 89.56 & \textbf{77.79} & 24.22 & 67.28 & 57.75 & 68.67 & 71.68 & 69.16 \\
        TCP & & 60.60 & 90.93 & 88.20 & 64.97 & 89.17 & 76.67 & 24.30 & 66.73 & 57.80 & 69.67 & 73.70 & 69.34 \\
        \rowcolor{Gray}
        CoMuCo (Ours) & & \textbf{62.97} & \textbf{91.63} & \textbf{89.20} & \textbf{69.33} & \textbf{90.17} & 77.17 & \textbf{33.77} & \textbf{67.60} & \textbf{60.77} & \textbf{80.37} & \textbf{74.70} & \textbf{72.52} \\
        \midrule
        Zero-Shot CLIP & \multirow{7}{*}{8-shot} & 58.18 & 86.29 & 85.77 & 55.61 & 66.14 & 77.31 & 17.28 & 58.52 & 42.32 & 37.56 & 61.46 & 58.77 \\
        CoOp & & 61.56 & 90.21 & 85.32 & 68.43 & 91.18 & 71.82 & 26.13 & 65.52 & 59.97 & 76.73 & 71.94 & 69.89 \\
        Tip-Adapter-F & & 64.15 & 91.70 & {88.28} & 69.51 & 91.00 & 77.90 & 30.62 & {69.23} & 62.93 & {79.33} & 74.76 & 72.67 \\
        TaskRes & & 64.03 & 92.23 & 87.07 & 70.57 & 94.30 & 76.90 & 29.50 & 68.70 & 64.23 & 78.07 & 74.77 & 72.76 \\
        LP++ & & 63.76 & 91.84 & 87.71 & 67.81 & 92.61 & \textbf{78.53} & 27.73 & 69.34 & 62.42 & 75.86 & 74.86 & 72.04 \\
        TCP & & 63.23 & 91.83 & 88.73 & 70.70 & 93.10 & 77.60 & 29.00 & 69.37 & 64.03 & 75.00 & 77.20 & 72.71 \\
        \rowcolor{Gray}
        CoMuCo (Ours) & & \textbf{64.37} & \textbf{92.80} & \textbf{90.43} & \textbf{76.33} & \textbf{94.47} & 78.13 & \textbf{43.60} & \textbf{70.20} & \textbf{65.73} & \textbf{83.90} & \textbf{79.30} & \textbf{76.30} \\
        \midrule
        Zero-Shot CLIP & \multirow{7}{*}{16-shot} &  58.18 & 86.29 & 85.77 & 55.61 & 66.14 & 77.31 & 17.28 & 58.52 & 42.32 & 37.56 & 61.46 & 58.77 \\
        CoOp & & 62.95 & 91.83 & 87.01 & 73.36 & 94.51 & 74.67 & 31.26 & 69.26 & 63.58 & 83.53 & 75.71 & 73.42 \\
        Tip-Adapter-F & & 65.44 & 92.63 & {88.18} & 75.75 & 94.23 & 78.11 & 35.86 & {71.00} & 66.94 & {84.94} & {79.03} & 75.65 \\
        TaskRes & & 64.75 & 92.90 & 88.10 & 74.93 & {96.10} & 78.23 & 33.73 & 70.30 & {67.57} & 82.57 & 76.87 & 75.10 \\
        LP++ & & 64.73 & 92.73 & 88.38 & 72.33 & 94.26 & 78.88 & 31.73 & 71.23 & 66.40 & 80.53 & 77.46 & 74.42 \\
        TCP & & 64.80 & 92.77 & 89.43 & 75.87 & 95.10 & 78.70 & 33.67 & 71.63 & 67.07 & 80.50 & 79.70 & 75.38 \\
        \rowcolor{Gray}
        CoMuCo (Ours) & & \textbf{66.27} & \textbf{94.17} & \textbf{90.87} & \textbf{85.07} & \textbf{97.60} & \textbf{79.43} & \textbf{55.63} & \textbf{72.33} & \textbf{69.83} & \textbf{89.40} & \textbf{82.87} & \textbf{80.32} \\
        \bottomrule
    \end{tabular}
    \end{adjustbox}
    \caption{Performance comparison on CLIP benchmark on ResNet50.
    }
    \label{tab:numerical_compare_CLIP}
\end{table*}

\begin{table*}[h]
    \centering
    \begin{adjustbox}{max width = 0.9\textwidth}
    \begin{tabular}{lccccccccccccc}
        \toprule
        Method & Setting & ImageNet & Caltech101 & OxfordPets & StanfordCars & Flowers102 & Food101 & FGVCAircraft & SUN397 & DTD & EuroSAT & UCF101 & Average \\
        \midrule
        CLIP~\cite{CLIP} & \multirow{6}{*}{1-shot} & 66.73 & 93.35 & 88.25 & 65.48 & 67.44 & 83.65 & 23.67 & 62.59 & 44.27 & 42.01 & 65.13 & 63.87  \\
        MaPLe~\cite{maple} & & 68.73 & 93.67 & 91.53 & 68.07 & 78.80 & 84.40 & 27.97 & 68.40 & 50.97 & 41.90 & 72.60 & 67.91 \\
        TCP~\cite{yao2024tcp} &  & 67.03 & 93.53 & 91.43 & 66.17 & \textbf{86.87} & 84.67 & 28.87 & 69.00 & 54.80 & 62.53 & 73.37 & 70.75 \\
        DePT~\cite{zhang2024dept} &  & 64.23 & 93.70 & 90.43 & 67.63 & 82.77 & 84.70 & \textbf{30.03} & 68.03 & 54.07 & 52.30 & 74.03 & 69.27 \\
        TextRefiner~\cite{xie2024textrefiner} &  & 69.00 & 92.13 & 88.17 & 65.07 & 70.10 & 85.40 & 25.63 & 67.97 & 44.83 & 53.93 & 69.93 & 66.56 \\
        SkipT~\cite{wu2025skip} &  & 69.20 & \textbf{93.87} & 91.60 & 69.63 & 83.63 & 85.67 & 29.93 & \textbf{69.10} & 54.50 & \textbf{72.23} & \textbf{75.30} & 72.24 \\
        \rowcolor{Gray}
        CoMuCo (Ours) &  & \textbf{69.97} & 93.70 & \textbf{91.87} & \textbf{69.93} & 86.53 & \textbf{86.03} & 29.00 & 68.93 & \textbf{56.53} & 69.70 & 75.20 & \textbf{72.49} \\
        
        \midrule
        CLIP~\cite{CLIP} & \multirow{6}{*}{2-shot} & 66.73 & 93.35 & 88.25 & 65.48 & 67.44 & 83.65 & 23.67 & 62.59 & 44.27 & 42.01 & 65.13 & 63.87  \\
        MaPLe~\cite{maple} &  & 69.47 & 94.20 & \textbf{92.63} & 69.80 & 84.47 & 85.03 & 30.93 & 70.53 & 55.63 & 72.30 & 74.30 & 72.66 \\
        TCP~\cite{yao2024tcp} &  & 68.30 & 94.13 & 90.57 & 71.00 & 90.87 & 85.17 & 32.23 & 71.03 & 58.43 & 70.63 & 77.70 & 73.64 \\
        DePT~\cite{zhang2024dept} &  & 65.97 & 94.27 & 90.90 & 71.77 & 88.80 & 85.30 & 32.43 & 70.37 & 59.83 & 68.57 & 77.60 & 73.25 \\
        TextRefiner~\cite{xie2024textrefiner} &  & 70.10 & 93.37 & 90.10 & 65.17 & 73.37 & 85.70 & 26.63 & 69.07 & 48.70 & 53.40 & 69.77 & 67.76 \\
        SkipT~\cite{wu2025skip} &  & 70.23 & \textbf{95.10} & 92.43 & 73.17 & 90.33 & 85.90 & \textbf{33.53} & 70.60 & 58.83 & \textbf{78.50} & 78.37 & 75.18 \\
        \rowcolor{Gray}
        CoMuCo (Ours) &  & \textbf{70.67} & 94.47 & 92.23 & \textbf{74.07} & \textbf{91.00} & \textbf{86.57} & 32.67 & \textbf{71.10} & \textbf{61.30} & 76.13 & \textbf{78.53} & \textbf{75.34} \\
        
        \midrule
        CLIP~\cite{CLIP} & \multirow{6}{*}{4-shot} & 66.73 & 93.35 & 88.25 & 65.48 & 67.44 & 83.65 & 23.67 & 62.59 & 44.27 & 42.01 & 65.13 & 63.87  \\
        MaPLe~\cite{maple} &  & 70.77 & 95.30 & 93.27 & 71.97 & 90.47 & 85.67 & 32.63 & 72.73 & 61.17 & 76.57 & 77.93 & 75.32 \\
        TCP~\cite{yao2024tcp} &  & 69.97 & 95.17 & 92.00 & 76.43 & 93.83 & 85.50 & 36.37 & 73.40 & 64.07 & 73.63 & 80.77 & 76.47 \\
        DePT~\cite{zhang2024dept} &  & 68.57 & 95.10 & 92.83 & 76.13 & 93.87 & 85.70 & 34.57 & 73.07 & 64.13 & 78.87 & 81.03 & 76.72 \\
        TextRefiner~\cite{xie2024textrefiner} &  & 70.70 & 93.13 & 92.57 & 67.90 & 76.17 & \textbf{87.03} & 29.43 & 70.80 & 51.37 & 56.00 & 71.87 & 69.72 \\
        SkipT~\cite{wu2025skip} &  & 71.40 & 95.60 & \textbf{93.33} & 77.60 & 94.27 & 86.00 & 39.90 & 73.07 & 65.70 & \textbf{83.40} & 82.53 & 78.44 \\
        \rowcolor{Gray}
        CoMuCo (Ours) &  & \textbf{71.70} & \textbf{95.97} & 93.17 & \textbf{79.90} & \textbf{95.30} & 86.80 & \textbf{39.93} & \textbf{73.50} & \textbf{67.33} & 82.67 & \textbf{82.57} & \textbf{78.98} \\
        
        \midrule
        CLIP~\cite{CLIP} & \multirow{6}{*}{8-shot} & 66.73 & 93.35 & 88.25 & 65.48 & 67.44 & 83.65 & 23.67 & 62.59 & 44.27 & 42.01 & 65.13 & 63.87  \\
        MaPLe~\cite{maple} & & 71.63 & 95.40 & 93.07 & 74.63 & 94.27 & 86.47 & 36.97 & 74.00 & 65.57 & 84.60 & 81.03 & 77.97 \\
        TCP~\cite{yao2024tcp} &  & 71.60 & 95.43 & 92.30 & 80.27 & 96.20 & 86.53 & 40.93 & 75.07 & 68.80 & 77.57 & 83.23 & 78.90 \\
        DePT~\cite{zhang2024dept} &  & 71.37 & 95.70 & 93.47 & 80.63 & 96.27 & 86.60 & 41.00 & 75.30 & 69.40 & 82.83 & 83.57 & 79.65 \\
        TextRefiner~\cite{xie2024textrefiner} &  & 70.77 & 95.17 & 92.70 & 69.87 & 84.07 & 86.97 & 31.00 & 72.17 & 55.23 & 59.57 & 75.37 & 72.08 \\
        SkipT~\cite{wu2025skip} &  & 72.40 & 95.90 & 93.40 & 82.33 & 96.60 & 86.73 & 46.50 & 74.77 & 69.77 & 86.57 & 85.60 & 80.96 \\
        \rowcolor{Gray}
        CoMuCo (Ours) &  & \textbf{73.20} & \textbf{96.00} & \textbf{93.50} & \textbf{85.77} & \textbf{97.37} & \textbf{87.47} & \textbf{52.00} & \textbf{75.87} & \textbf{71.30} & \textbf{86.80} & \textbf{85.93} & \textbf{82.29} \\
        
        \midrule
        CLIP~\cite{CLIP} & \multirow{6}{*}{16-shot} & 66.73 & 93.35 & 88.25 & 65.48 & 67.44 & 83.65 & 23.67 & 62.59 & 44.27 & 42.01 & 65.13 & 63.87  \\
        MaPLe~\cite{maple} &  & 72.57 & 95.67 & 93.30 & 78.00 & 95.80 & 87.37 & 40.63 & 75.47 & 70.50 & 89.03 & 82.70 & 80.09 \\
        TCP~\cite{yao2024tcp} &  & 72.40 & 95.83 & 92.77 & 84.00 & 97.43 & 87.23 & 46.13 & 76.80 & 72.70 & 85.00 & 85.40 & 81.43 \\
        DePT~\cite{zhang2024dept} &  & 73.35 & 96.27 & 94.13 & 85.03 & 97.83 & 87.30 & 47.00 & 77.30 & 74.03 & 91.43 & 85.57 & 82.66 \\
        TextRefiner~\cite{xie2024textrefiner} &  & 70.90 & 95.27 & 93.13 & 72.10 & 92.37 & 87.23 & 34.37 & 74.33 & 61.93 & 64.30 & 79.63 & 75.05 \\
        SkipT~\cite{wu2025skip} &  & 73.83 & 96.77 & 94.00 & 87.43 & 98.17 & 87.13 & 55.57 & 76.80 & 74.07 & \textbf{91.57} & 87.70 & 83.91 \\
        \rowcolor{Gray}
        CoMuCo (Ours) &  & \textbf{74.90} & \textbf{96.87} & \textbf{94.47} & \textbf{90.30} & \textbf{98.67} & \textbf{88.00} & \textbf{64.30} & \textbf{77.60} & \textbf{76.77} & 91.47 & \textbf{88.37} & \textbf{85.61} \\
        \bottomrule
    \end{tabular}
    \end{adjustbox}
    \caption{Performance comparison on CLIP benchmark on ViT-B/16.
    }
    \label{tab:numerical_compare_CLIP_vit}
\end{table*}

\begin{figure}[htb]
    \centering
    \begin{subfigure}{0.9\linewidth}
        \centering
        \includegraphics[width=1\linewidth]{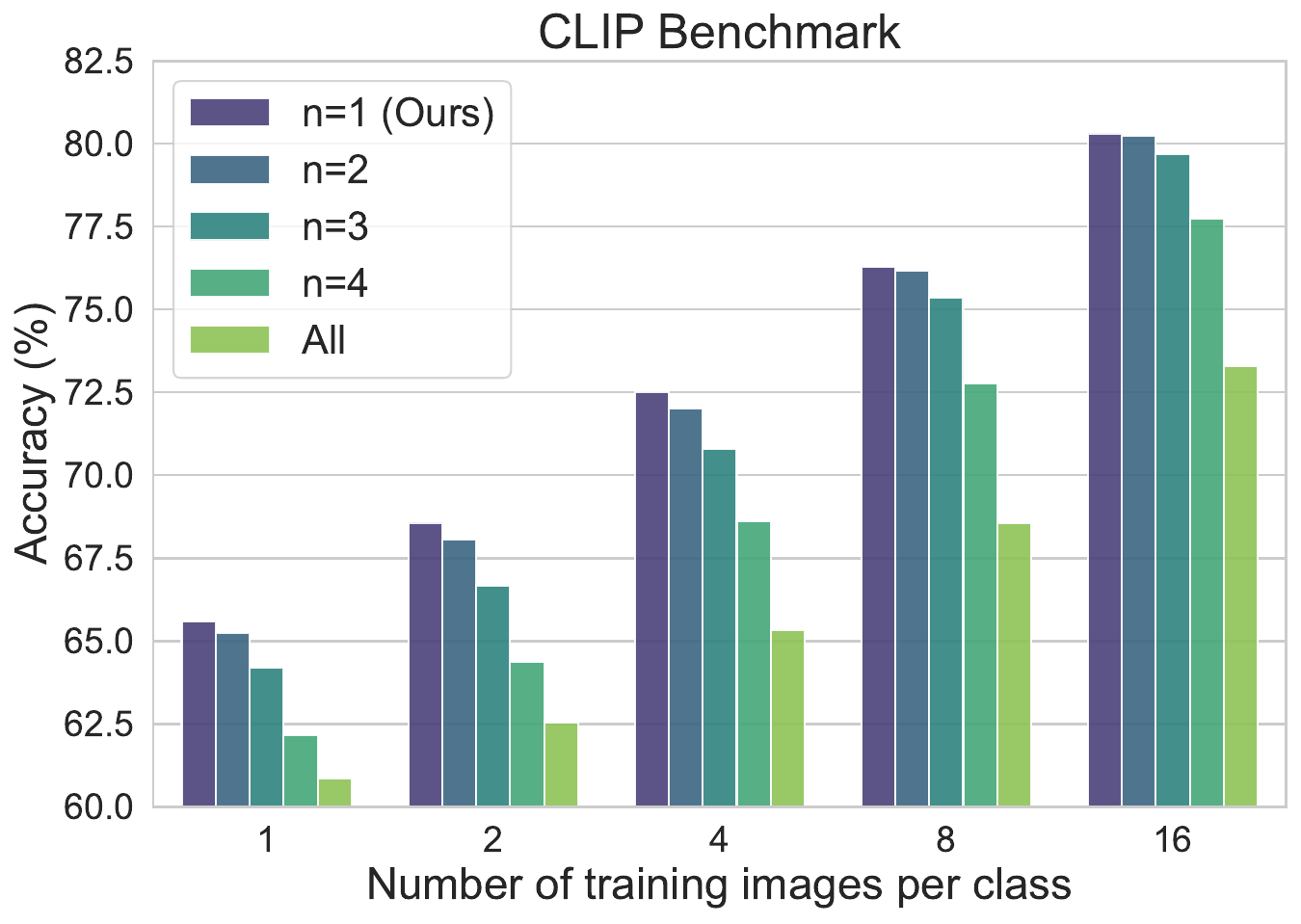}
        \label{fig:clip_fap}
    \end{subfigure}
    \begin{subfigure}{0.9\linewidth}
        \centering
        \includegraphics[width=1\linewidth]{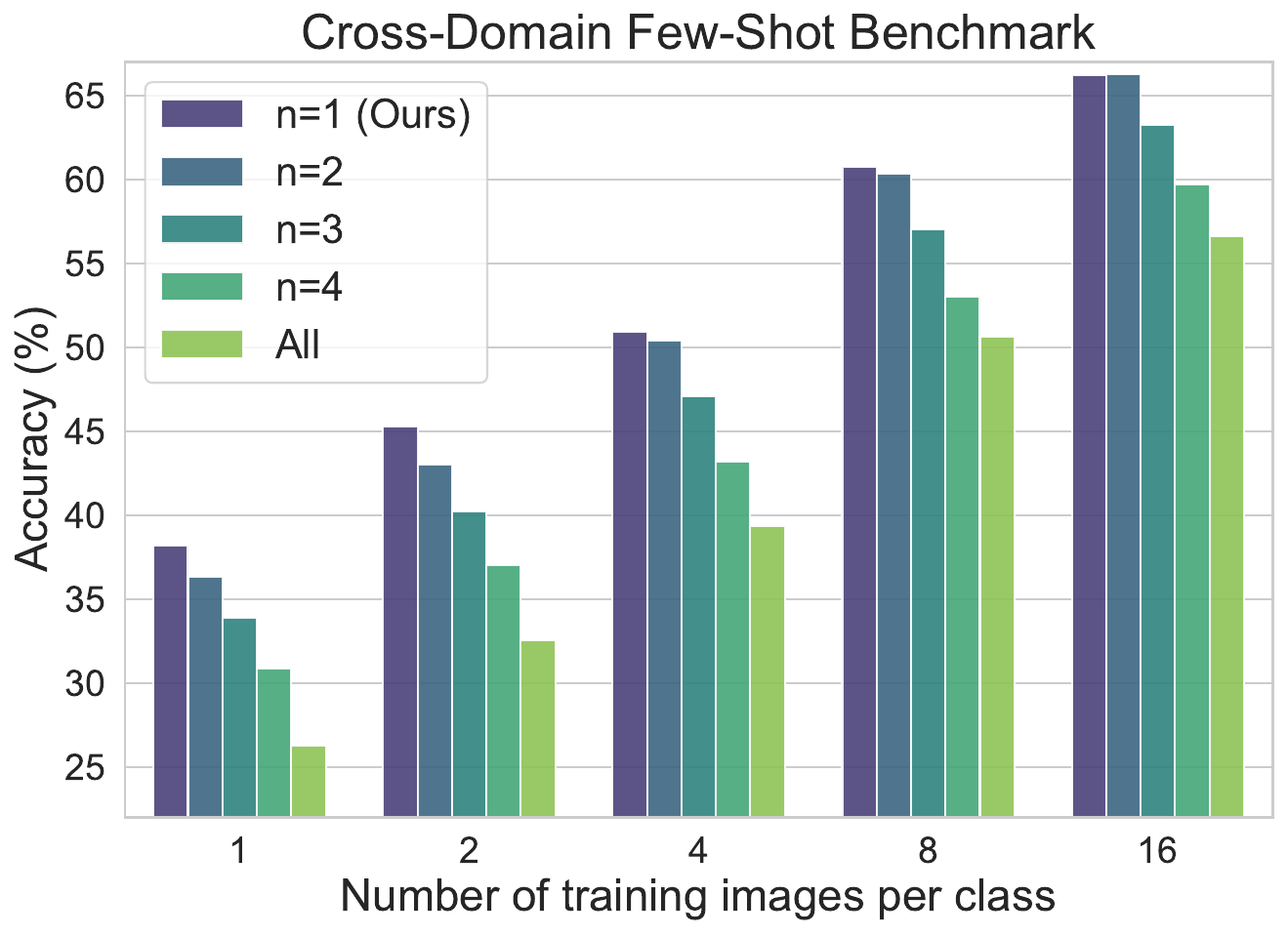}
        \label{fig:cross_fap}
    \end{subfigure}
    \caption{Average performance on the CLIP benchmark and cross-domain few-shot benchmark with various settings of fine-tuned layers for FR. `n' represents the number of tuned layers close to the model output, with `n=1' being the selected setup in the proposed CoMuCo and `All' fine-tuning the entire visual encoder.
    }
    \label{fig:fap_full}
\end{figure}

\section{C. More Results and Details}
\label{app:results}

\subsection{C.1 Software and hardware}
All methods are implemented in Pytorch2.3.1. We run all the experiments on NVIDIA RTX3090 GPU.

\subsection{C.2 Results on Cross-Domain Few-Shot Benchmark}
\label{app:results_cross}
The full numerical results on the cross-domain few-shot benchmark are presented in \cref{tab:numerical_compare_cross}. 
While zero-shot CLIP shows moderate success on NWPU-RESISC45 due to its natural-like elements (e.g., rivers, houses), it fails to generalize well on more semantically divergent datasets.
As shown in \cref{tab:numerical_compare_cross}, our method consistently outperforms all baselines on ResNet50 across [1, 2, 4, 8, 16]-shot settings, with the most notable gains in the 8- and 16-shot scenarios. Specifically, average improvements of 0.9\%, 2.25\%, 2.78\%, 5.27\%, and 7.04\% were observed over the best-performing alternatives.
Using ViT-B/16 as the visual encoder (\cref{tab:numerical_compare_cross_vit}), our method further demonstrates superior performance. Overall, the results confirm CoMuCo's strong generalization ability across diverse domains.

\subsection{C.3 Results on CLIP Benchmark}
\label{app:results_clip}
The complete numerical results of our method and the comparative methods evaluated on the CLIP benchmark are presented in \cref{tab:numerical_compare_CLIP} and \cref{tab:numerical_compare_CLIP_vit}.
It is evident that our method consistently surpasses the compared methods in terms of average performance across all [1, 2, 4, 8, 16]-shot settings, with the performance gap widening as the sample size increases on ResNet50. Specifically, compared to the best competing methods, average performance improvements of 1.55\%, 1.56\%, 2.82\%, 3.54\%, and 4.67\% were observed under different settings. 
In summary, the evaluation results on the CLIP benchmark suggest that CoMuCo exhibits robust learning and generalization capabilities in various scenarios, effectively addressing classification tasks in few-shot settings.

\subsection{C.4 Impact of Fine-tuning Layer Configurations in Feature Refiner}
\label{app:results_fap}
To evaluate the effects of the number of fine-tuning layers on FR, experiments were performed with various fine-tuning layer configurations across 
11 datasets in the CLIP Benchmark and 7 datasets in the Cross-Domain Few-Shot Benchmark. 
As shown in \cref{fig:fap_full}, more fine-tuning layers result in larger performance degradation, particularly in scenarios with limited samples or when adapting to cross-domain datasets. Specifically, under 4-shot conditions, additional layers reduced performance by 0.49\%, 1.73\%, 3.88\%, and 7.19\% on the CLIP Benchmark, and by 0.56\%, 3.86\%, 7.78\%, and 11.58\% on the cross-domain Benchmark. Under the more constrained 1-shot setting on cross-domain data, performance declines of 1.86\%, 4.28\%, 7.32\%, and 11.90\% were observed.
These results suggest that, with extremely limited data, shallow layers are more prone to overfitting, leading to decreased generalization performance. This effect is exacerbated with fewer samples and cross-domain learning, where model adjustments are more substantial and more likely to fit noise, resulting in further performance reduction. Therefore, fine-tuning only the deeper layers helps to retain valuable pre-trained knowledge while alleviating overfitting.

\subsection{C.5 Comparison with ensemble models}
Our framework is fundamentally distinct from simple model ensembles, with its inter-model interaction mechanism proven to significantly boost performance. 
Experiment were performed by ensembling TCP, TaskRes, and CLIP. The results in Tab.1 in manuscript (Rows 2, 4, 6 \& 9) and \cref{tab:ensemble} indicate that our method's superiority stems from more than model ensembling and outperforms the ensemble of competing methods.

\begin{table}[h]
\centering
\resizebox{0.45\textwidth}{!}{
\begin{tabular}{cccc}
\toprule
 & \textbf{ImageNet} & \textbf{Stanford Cars} & \textbf{FGVC Aircraft} \\
\midrule
\textbf{TCP+TaskRes+CLIP} & 64.54 & 75.00 & 32.34 \\
\textbf{CoMuCo (Ours)} & \textbf{66.27} & \textbf{85.07} & \textbf{55.63} \\
\rowcolor{Gray}
$\Delta$ & {+1.73} & {+10.07} & {+23.29} \\
\bottomrule
\end{tabular}
}
\caption{Performance comparison with ensemble models.}
\label{tab:ensemble}
\end{table}

\begin{table}[tb]
    \centering
    \small
    \begin{adjustbox}{max width = 0.45\textwidth}
    \begin{tabular}{lcccc}
        \toprule
        \multirow{2}{*}{Method} & \multicolumn{3}{c}{Datasets} \\
        \cmidrule(lr){2-4}
        & ImageNet & Stanford Cars & Galaxy10 DECaLS\\
        \midrule
        Zero-Shot & 66.10 & 86.30 & 14.30 \\
        Tip-Adapter-F & 70.18 & 90.90 & 52.70 \\
        TCP & 69.83 & 90.77 & 51.53 \\
        \rowcolor{Gray}
        CoMuCo (Ours) & \textbf{71.07} & \textbf{91.93} & \textbf{56.00} \\
        \rowcolor{Gray}
        $\Delta$ & \textbf{0.89} & \textbf{1.03} & \textbf{3.30} \\
        \bottomrule
    \end{tabular}
    \end{adjustbox}
    \caption{Performance on ConvNext-Base Model, with $\Delta$ indicating the superiority of the proposed CoMuCo framework over the strongest baseline method.}
    \label{tab:convnext}
\end{table}

\subsection{C.6 Extension to Non-transformer Model}
The architectural generalization of CoMuCo is demonstrated through its application to ConvNeXt-Base visual encoders from OpenCLIP~\cite{cherti2023reproducible}, pretrained on LAION-400M~\cite{laion400m}. Given the absence of attention pooling in ConvNeXt-Base, the final block of its fourth stage is designated for FI training, while the entire fourth stage is used for FR training. 
The results (\cref{tab:convnext}) demonstrate consistent performance advantages, with accuracy improvements of 0.89\% on ImageNet, 1.03\% on Stanford Cars, and a substantial 3.3\% gain on Galaxy10 DECaLS under 16-shot conditions.
These results confirm CoMuCo’s effectiveness across diverse architectures and its strong learning performance in cross-domain applications.

\begin{table}[htbp]
  \centering
  \resizebox{0.45\textwidth}{!}{
    \begin{tabular}{lcccc}
      \toprule
      \multirow{2}{*}{Method} & \multicolumn{2}{c}{FPS} & \multirow{2}{*}{\begin{tabular}[c]{@{}c@{}}Average Accuracy (\%)\\ on CLIP Benchmark\end{tabular}} \\
      \cmidrule(r){2-3}
      & Training & Inference & \\
      \midrule
      CoCoOp & 4 & 13 & 74.9 \\
      TCP & 121 & 622 & 81.43 \\
      TextRefiner & 85 & 534 & 75.05 \\
      \rowcolor{Gray}
      CoMuCo (Ours) & 251 & 542 & 84.33 \\
      \bottomrule
    \end{tabular}
  }
  \caption{Training and Inference Efficiency Comparison. Inference time tested on SUN397 dataset using ViT-B/16 as the visual encoder.}
  \label{tab:efficiency_comparison}
\end{table}

\subsection{C.7 Training and Inference Efficiency}
As shown in \cref{tab:efficiency_comparison}, CoMuCo significantly enhances classification accuracy while maintaining high training and inference efficiency. Specifically, CoMuCo demonstrates a substantial advantage in training speed (251 FPS), outperforming all three baselines by a large margin, which highlights its scalability and practical applicability across various training scenarios. In terms of inference efficiency, CoMuCo approaches the performance of TCP (542 vs. 622 FPS), matches that of TextRefiner (534 FPS), and significantly surpasses CoCoOp (13 FPS), while achieving markedly higher accuracy than all three. These comparisons suggest that CoMuCo achieves a robust compromise between model expressiveness and computational viability across both training and inference phases.

\section{D. Visualization Analysis}
\label{app:results_vis}
\subsection{D.1 Visualization of Attention Regions}

To further investigate CoMuCo, a visual analysis was conducted along its dual modules. Specifically, several example images were randomly selected from the test sets of the ImageNet, Stanford Cars, and Galaxy datasets. Using GradCAM~\cite{selvaraju2017grad}, the attention regions of the models were observed for each test image and its corresponding class text. As illustrated in \cref{fig:vis_imagenet}, \cref{fig:vis_cars} and \cref{fig:vis_galaxy}, when the original CLIP model failed to correctly locate the target corresponding to the textual category, the adapted FI and FR often succeeded in identifying the object of interest, as demonstrated by the examples of ``mongoose" and ``meerkat" in \cref{fig:vis_imagenet}. Conversely, when CLIP successfully located the target, FI and FR tended to capture more comprehensive attention regions, as shown in the examples of ``fig" and ``pillow" in \cref{fig:vis_imagenet}.

Comparison of FI and FR reveals that, when dealing with natural images, FI sometimes better disregards distracting elements in the image. For example, in the cases of ``2012 Hyundai Accent Sedan" and ``2012 Rolls-Royce Phantom Sedan" (\cref{fig:vis_cars}), FI focused more precisely on the main body of the car compared to FR. In contrast, when handling cross-domain data, FR was more effective at identifying regions associated with the textual category. For instance, in the examples of ``Edge-on Galaxies with Bulge" and ``Edge-on Galaxies without Bulge" (\cref{fig:vis_galaxy}), FR's attention is more concentrated on the ``bulge" regions than FI.

In summary, the fine-tuned FI and FR demonstrated an enhanced ability to focus on relevant aspects of the target datasets compared to the original CLIP vision encoder. FI exhibited reduced overfitting by effectively ignoring distractions in natural images, while FR demonstrated superior discriminative capacity when addressing cross-domain data.

\subsection{D.2 Visualization of Learned Features}

Feature visualization was conducted using ImageNet, Stanford Cars, and Galaxy10 DECaLS to examine the feature extraction behaviors of the model’s dual modules. From ImageNet and Stanford Cars, 20 categories were randomly chosen, while all 10 categories from Galaxy10 DECaLS were included. Test images were processed through the original CLIP encoder, FI, and FR, with extracted features visualized via t-SNE (\cref{fig:tsne}).

The results show that FI and FR consistently outperform the CLIP encoder in intra-class compactness and classification boundary clarity. However, their behaviors vary across datasets. In ImageNet, both FI and FR produce similar feature distributions. For Stanford Cars, FR further tightens intra-class clusters and increases inter-class separation, whereas FI improves feature discriminability but retains a distribution pattern closer to CLIP. In Galaxy10 DECaLS, where CLIP fails to distinguish categories, FI successfully clusters intra-class samples but struggles with inter-class separation, while FR excels in both aspects.

These results suggest that both FI and FR enhance intra-class clustering, yet FI tends to preserve the original feature distribution, making it more suitable for domains similar to the pre-training dataset. Conversely, FR exhibits superior performance in cross-domain scenarios by improving intra-class clustering and inter-class separation. This highlights the distinct functionalities of these two modules: FI extracts and refines knowledge relevant to the downstream task from pre-trained representations, whereas FR facilitates deep domain adaptation by further fine-tuning the feature extractor, enabling the acquisition of novel knowledge beyond pre-training through downstream data.

\begin{figure}[!h]
    \centering
    \includegraphics[width=\linewidth]{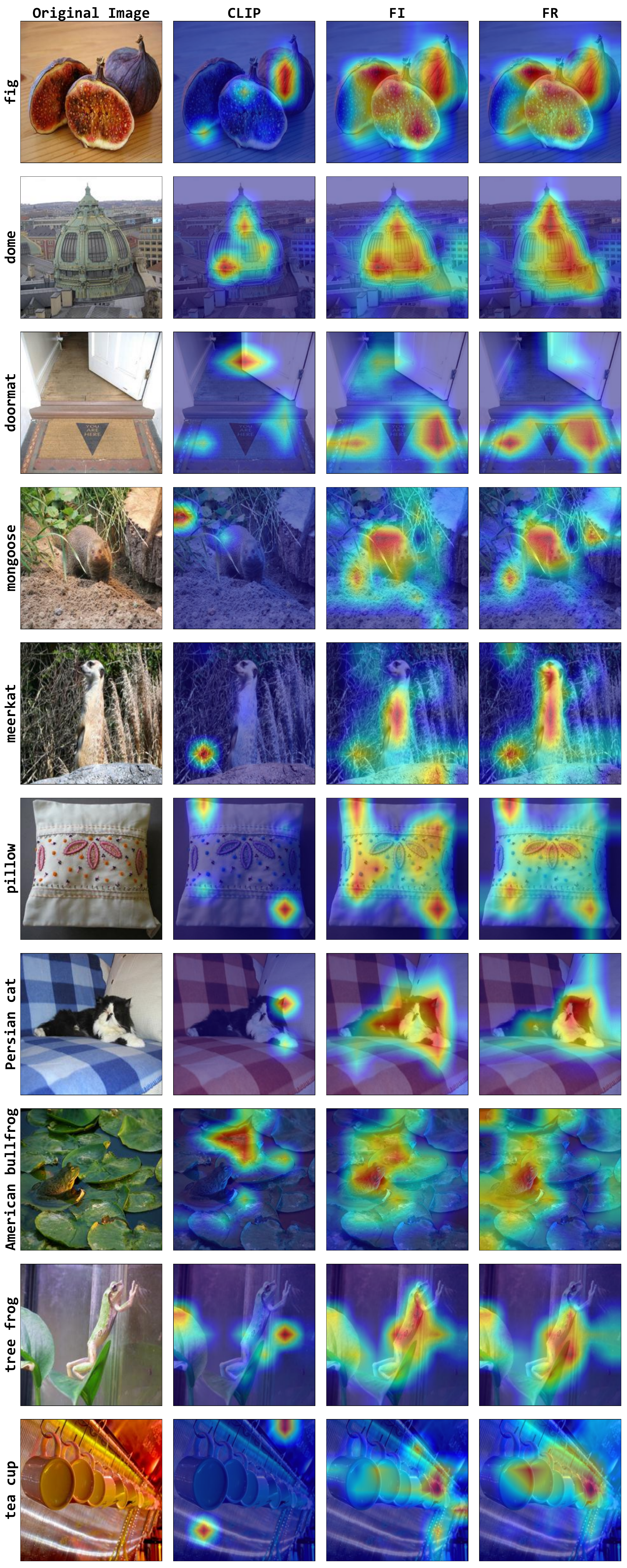}
            
    \caption{Visualization of visual features for exemplar images from ImageNet by GradCAM. From left to right: original images, GradCAM heatmaps overlaid on input images respectively from the CLIP visual encoder, the FI, and the FR. Warmer colors indicate higher attention.
    }
    \label{fig:vis_imagenet}
\end{figure}

\begin{figure}[!h]
    \centering
    \includegraphics[width=\linewidth]{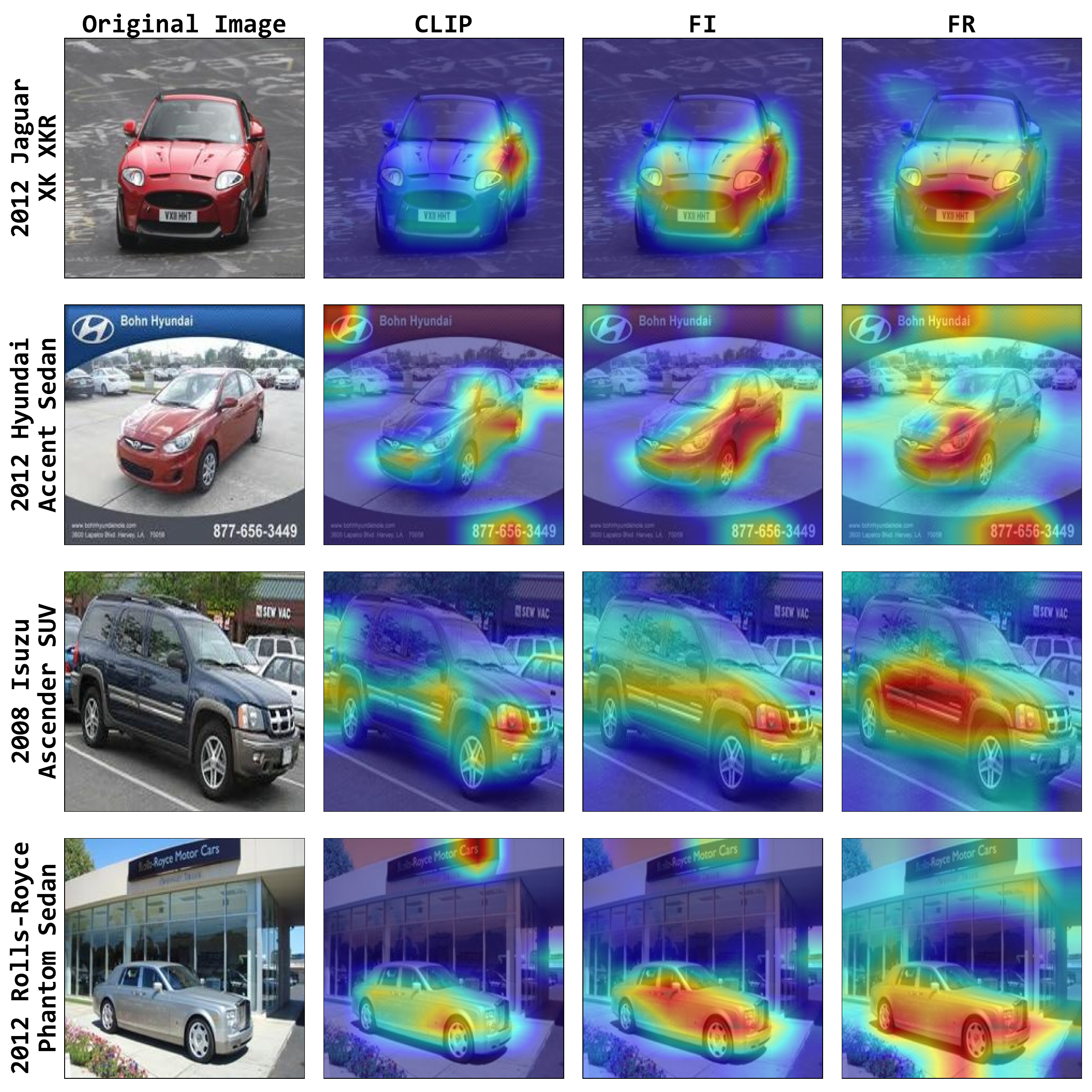}
            
    \caption{Visualization of visual features for exemplar images from StanfordCars by GradCAM. From left to right: original images, GradCAM heatmaps overlaid on input images respectively from the CLIP visual encoder, the FI, and the FR. Warmer colors indicate higher attention.
    }
    \label{fig:vis_cars}
\end{figure}

\begin{figure}[!h]
    \centering
    \includegraphics[width=\linewidth]{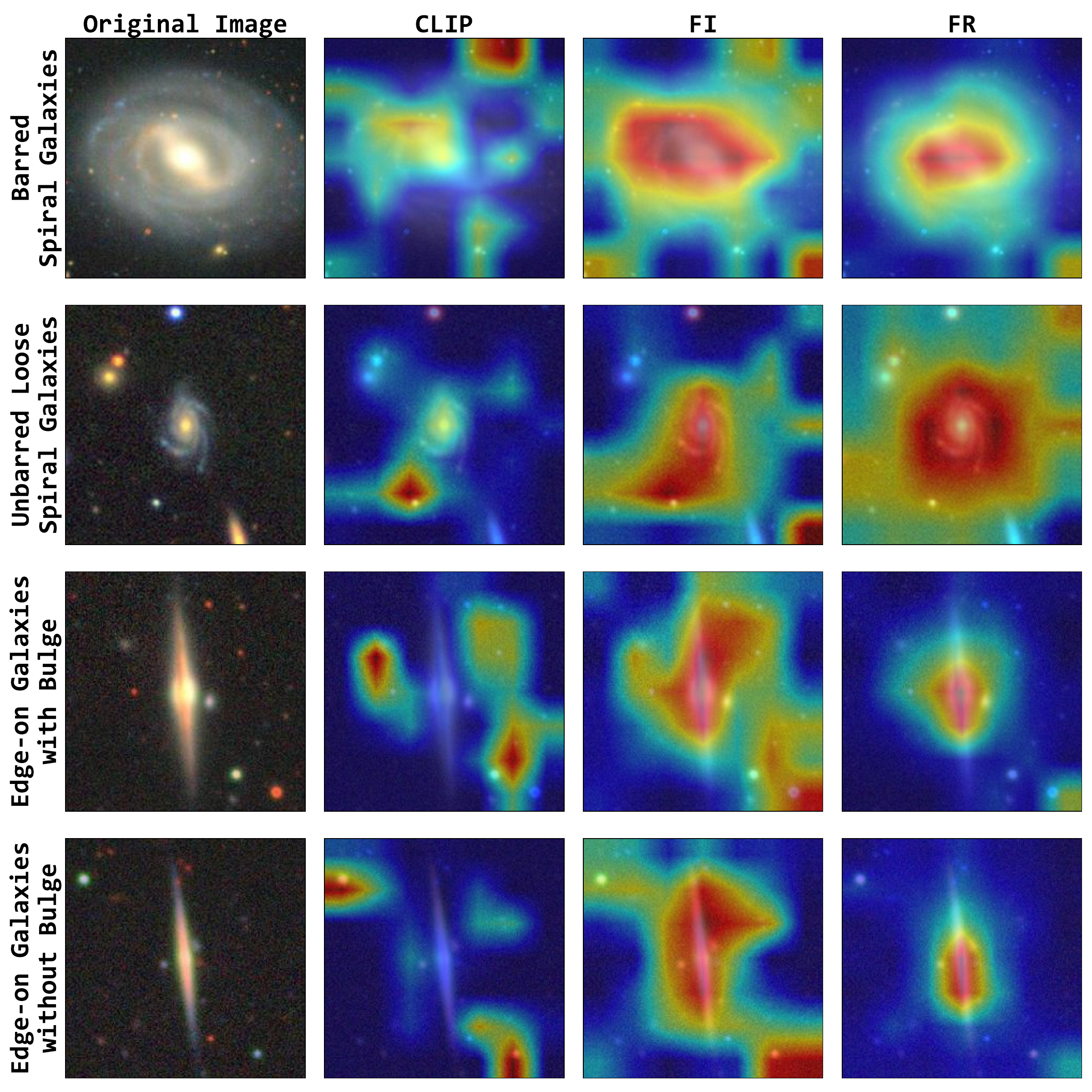}
            
    \caption{Visualization of visual features for exemplar images from Galaxy10 DECaLS by GradCAM. From left to right: original images, GradCAM heatmaps overlaid on input images respectively from the CLIP visual encoder, the FI, and the FR. Warmer colors indicate higher attention.
    }
    \label{fig:vis_galaxy}
\end{figure}

\begin{figure*}[!htb]
    \centering
    \begin{subfigure}[b]{\linewidth}
         \centering
         \includegraphics[width=\linewidth]{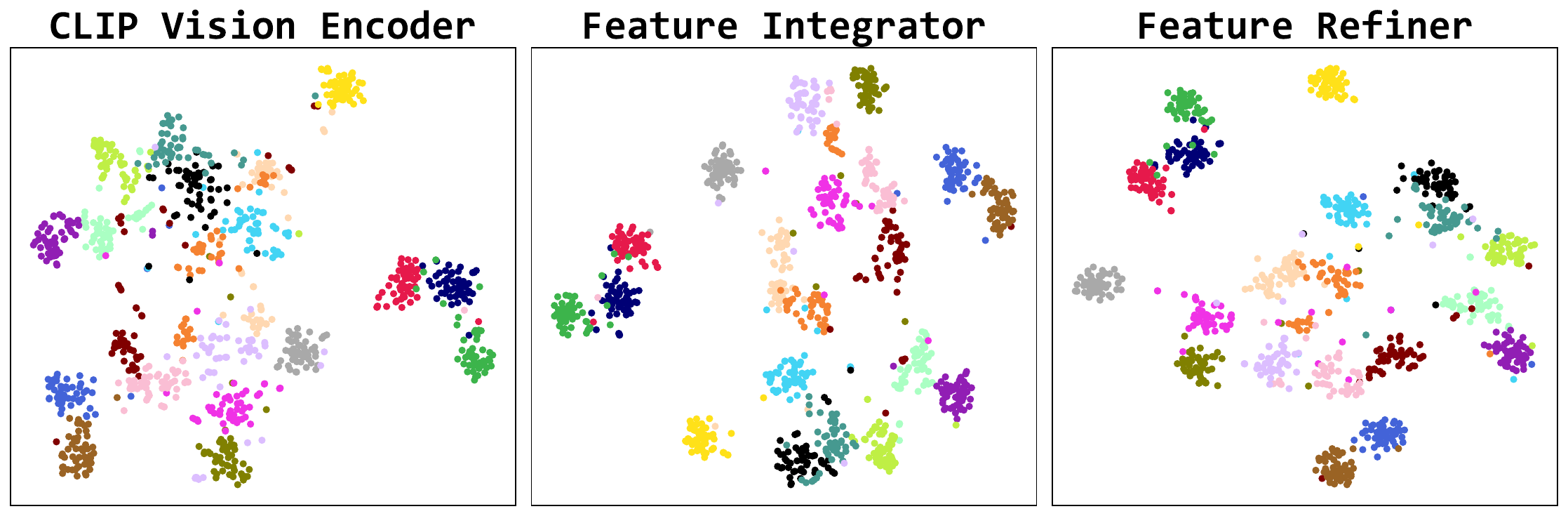}
         \caption{ImageNet}
         \label{fig:tsne-imagenet}
     \end{subfigure}

     \begin{subfigure}[b]{\linewidth}
         \centering
         \includegraphics[width=\linewidth]{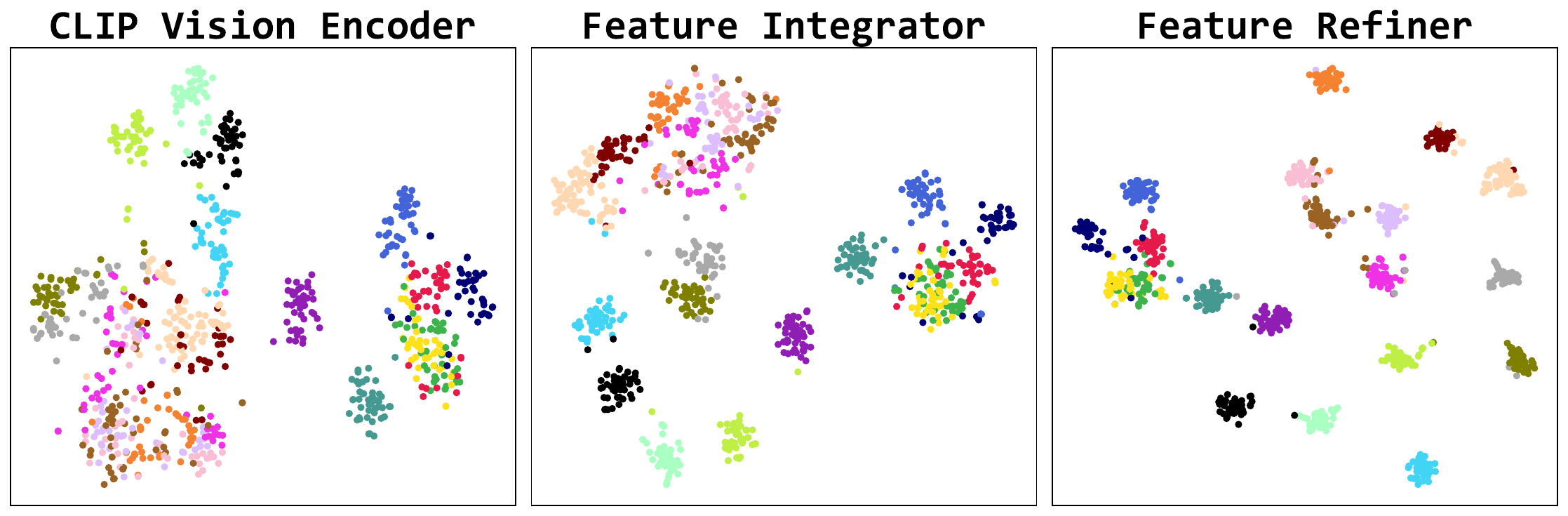}
         \caption{Stanford Cars}
         \label{fig:tsne-cars}
     \end{subfigure}

     \begin{subfigure}[b]{\linewidth}
         \centering
         \includegraphics[width=\linewidth]{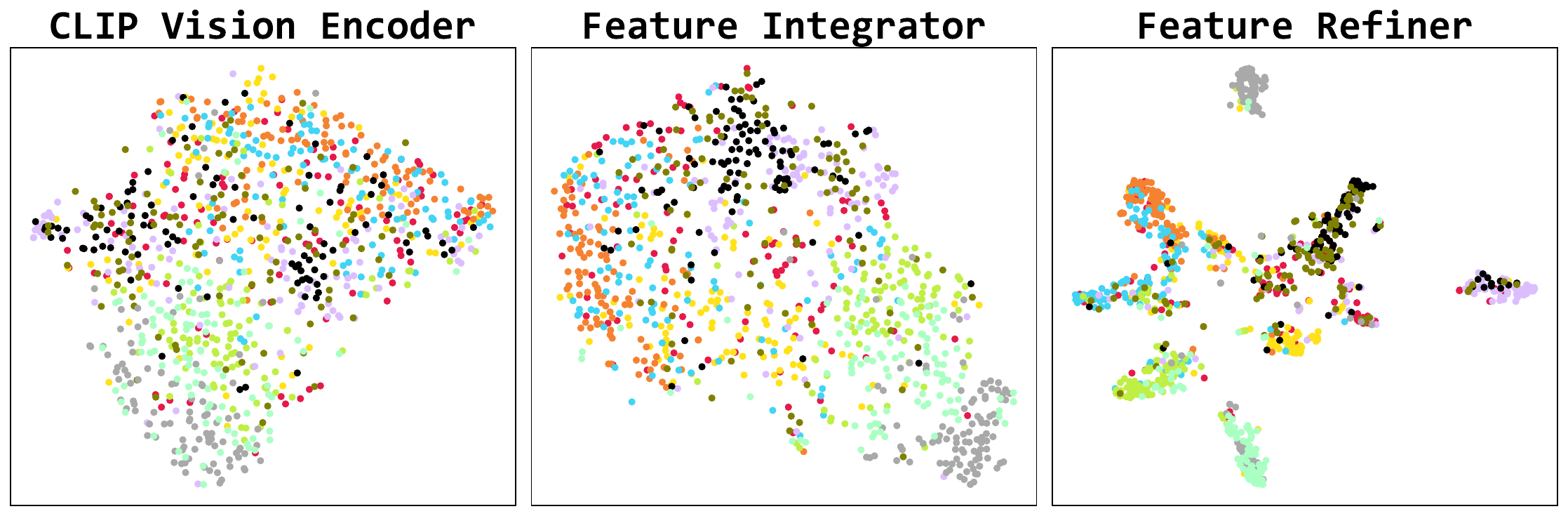}
         \caption{Galaxy10 DECaLS}
         \label{fig:tsne-galaxy}
     \end{subfigure}
            
    \caption{t-SNE visualization of vision feature. Dots in different colors stand for embeddings of different categories. From left to right, three distributions indicate the feature of original CLIP vision encoder, FI and FR, respectively.
    }
    \label{fig:tsne}
\end{figure*}